\documentclass{article}

 \usepackage[main, final]{neurips_2026}

\usepackage[utf8]{inputenc} 
\usepackage[T1]{fontenc}    
\usepackage{hyperref}       
\usepackage{url}            
\usepackage{booktabs}       
\usepackage{amsfonts}       
\usepackage{nicefrac}       
\usepackage{microtype}      
\usepackage{xcolor}         

\usepackage{amsmath}
\usepackage{amsthm}
\usepackage{wrapfig}
\usepackage{multirow}
\usepackage[table]{xcolor}
\usepackage{algorithm}
\usepackage{algpseudocode}
\usepackage{adjustbox}

\newcommand{\tr}{\operatorname{tr}}
\newcommand{\tc}{\operatorname{TC}}
\newcommand{\cca}{\operatorname{CCA}}
\newcommand{\mcp}{\operatorname{MCP}}
\newcommand{\qap}{\operatorname{QAP}}
\newcommand{\ourmethod}{UCCA}
\newtheorem{definition}{Definition}
\newtheorem{theorem}{Theorem}
\newtheorem{assumption}{Assumption}

\newtheorem{proposition}{Proposition}
\newtheorem{lemma}{Lemma}
\definecolor{lightgreen}{HTML}{E2EFDA}

\usepackage[most]{tcolorbox}

\newtcolorbox{takeaway}[1]{
    enhanced,
    colback=blue!3,                   
    colframe=blue!50!gray,            
    width=\textwidth,
    arc=3mm,
    boxrule=1pt,
    halign=center,                    
    fontupper=\itshape,               
    before skip=1em,
    after skip=1em,
    attach boxed title to top left={xshift=5mm, yshift=-3.5mm},
    boxed title style={
        colback=white,                
        colframe=blue!50!gray,        
        arc=1mm,
        boxrule=0.5pt,                
    },
    title={Takeaway #1},
    fonttitle=\sffamily\small\bfseries\color{blue!50!gray} 
}

\title{Unpaired Canonical Correlation Analysis}

\author{%
    Nir Ben-Ari\\
	Department of Computer Science and AI\\
	Bar-Ilan University\\
	Ramat-Gan, Israel \\
	\texttt{nirnirba@gmail.com} \\
	\AND
    Ronen Talmon \\
    Electrical and Computer Engineering \\
    Technion \\
    Haifa, Israel\\
    \texttt{ronen@ee.technion.ac.il} \\
    \And
	Uri Shaham \\
	Department of Computer Science and AI\\
	Bar-Ilan University\\
	Ramat-Gan, Israel \\
	\texttt{uri.shaham@biu.ac.il} \\
}

\begin{document}

\maketitle

\begin{abstract}
  Canonical Correlation Analysis (CCA) is a fundamental method for multiview shared space learning. However, its strict reliance on paired data poses a significant limitation, as such data is often difficult to obtain or entirely unavailable. In this paper, we present Unpaired CCA (\ourmethod), a novel method that learns linear projections to maximize the correlation of the true underlying pairing without access to any paired samples during training. We first establish theoretical results connecting the Quadratic Assignment Problem (QAP) to CCA. Leveraging these theoretical insights, we derive a practical method to maximize correlation exclusively from unpaired data. To the best of our knowledge, \ourmethod\ is the first approach to learn maximally correlated projections in a strictly unpaired setting. We validate \ourmethod\ on real-world multi-modal datasets, demonstrating that it significantly outperforms recent unpaired alignment baselines in recovering the underlying true correlation. This work fills a critical gap between traditional statistical multiview learning and the growing field of unpaired data learning.\footnote{\textbf{Code:} \url{https://github.com/shaham-lab/UCCA},\ \textbf{Project Page:} \url{https://shaham-lab.github.io/UCCA_page}}
\end{abstract}

\section{Introduction}

Many real-world applications rely on the integration of multiple data modalities \cite{baltruvsaitis2018multimodal}. One of the most foundational methods for multiview shared-space learning is Canonical Correlation Analysis (CCA) \cite{hotelling1935most}. That is, the extraction of maximally correlated components across views. CCA underpins a wide array of modern applications, including multimodal representation learning \cite{zbontar2021barlow, li2024csa}, multi-omics data integration \cite{butler2018integrating, jiang2023canonical, pusa2024stable}, and neuroimaging \cite{ju2024deep}. Although it has inspired various extensions, most notably multi-view \cite{rupnik2010multi, luo2015tensor} and non-linear \cite{andrew2013deep, wong2021deep} variants, CCA remains constrained by a critical bottleneck: a rigid reliance on strictly paired corresponding samples.

In the broader landscape of machine learning, the fundamental difficulty of acquiring paired data has sparked massive interest in unpaired settings \cite{yacobi2025sue, timilsina2024identifiable, gupta2025better, roschmann2026sotalign, zhang2025sail}. While large multimodal models show remarkable success in various applications \cite{radford2021learning, xu2023multimodal}, they rely heavily on massive amounts of paired data (i.e., images and their corresponding captions). In many modalities, such pairs are often prohibitively expensive, if not impossible, to obtain. In contrast, unpaired data is abundantly accessible, making the utilization of independent, unpaired samples a highly attractive alternative.

To address this fundamental limitation, we propose Unpaired CCA (UCCA). To the best of our knowledge, this is the first extension of CCA designed to operate entirely without paired data. Specifically, prior to this work, no method has successfully demonstrated or theoretically established the ability to capture true data correlations when correspondence is entirely absent during training.

\begin{wrapfigure}{r}{0.55\textwidth} \label{fig:open}
    \vspace{-10pt}
    \centering
    \includegraphics[width=0.5\textwidth]{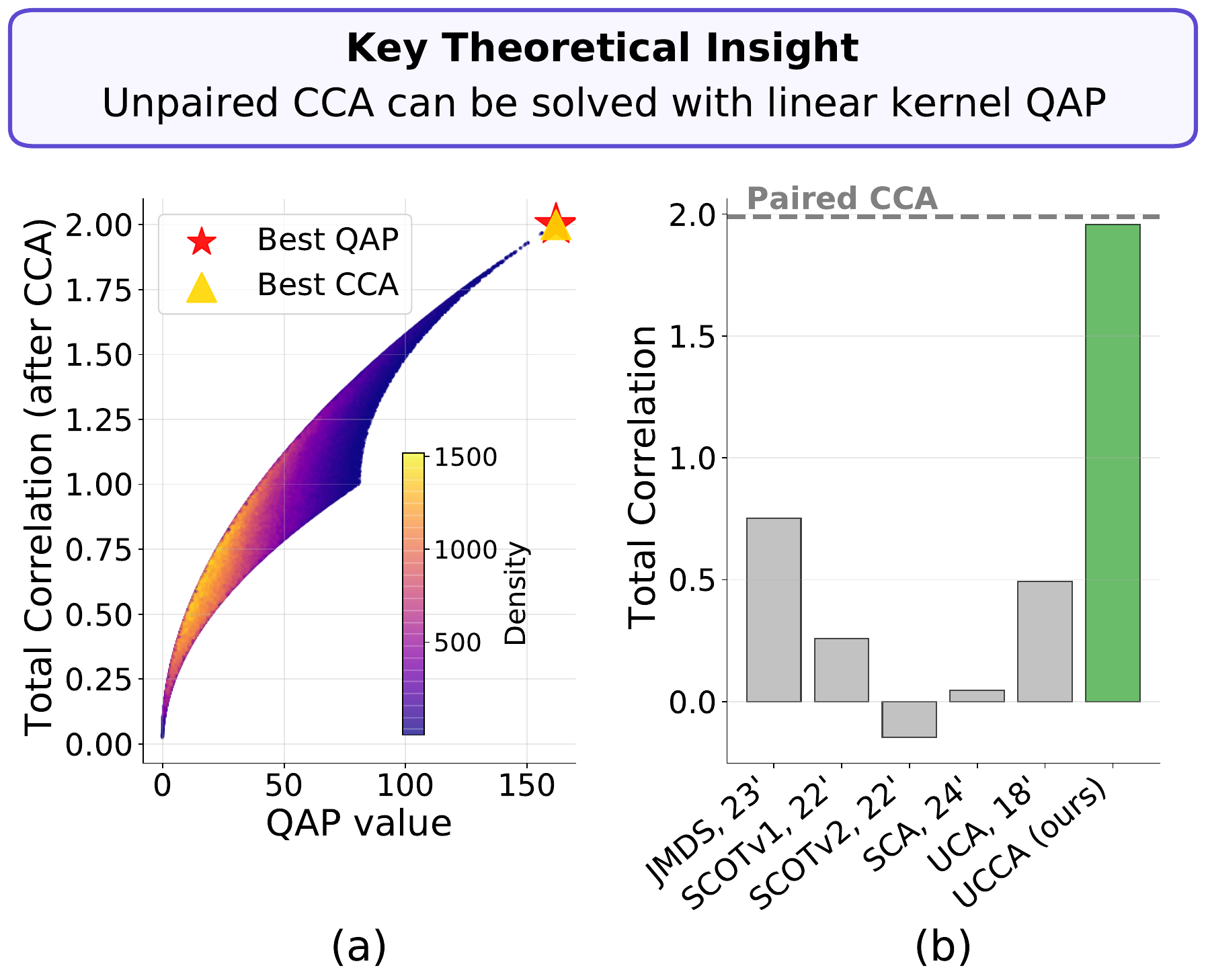}
    \vspace{-5pt}
    \caption{\textbf{The theoretical and empirical alignment of Unpaired CCA and QAP.} (a) QAP vs. CCA: On a 10-sample synthetic dataset, evaluating all 10! possible pairings reveals a direct positive relationship between the QAP value and total correlation (App. \ref{app:qap_cca}). (b) UCCA Performance: On the Handwritten benchmark (Sec. \ref{sec:results}), \ourmethod\ captures high correlation despite training fully unpaired, significantly outperforming unpaired alignment baselines and approaching the paired CCA upper bound.}
    \vspace{-20pt}
\end{wrapfigure}

We approach the Unpaired CCA problem from a theoretical perspective. We begin with a theoretical investigation, defining a proxy pairing for unpaired CCA scenarios. We then prove that this pairing is attainable by formulating the task as a special case of the Quadratic Assignment Problem (QAP) \cite{koopmans1957assignment}, enabling the use of known approximate efficient solvers. This framework yields strong theoretical guarantees, as summarized in Fig. \ref{fig:open}. Furthermore, Fig. \ref{fig:open}a highlights the direct positive relationship between QAP optimization and CCA correlation, substantiating our theoretical claims.

Derived directly from these theoretical insights, we introduce our Unpaired CCA (UCCA) method, a practical and effective tool for recovering the correlation components of ground-truth pairings using strictly unpaired training data. As demonstrated in Fig. \ref{fig:open}b, UCCA is currently the only method capable of learning highly correlated components in this setting, even achieving performance that is comparable to the paired CCA.

\paragraph{Contributions.} Our main contributions are: (1) We establish a theoretical connection between Unpaired CCA and the QAP; (2) We derive a practical method for solving Unpaired CCA; and (3) Our empirical results show that our method significantly reduces the gap to paired CCA, while operating entirely without paired data, achieving superior performance compared to existing approaches.

\section{Related Work}

\paragraph{CCA Extensions.} Canonical Correlation Analysis has been extensively adapted for diverse applications, yielding multi-view \cite{rupnik2010multi, luo2015tensor} and non-linear \cite{andrew2013deep, wong2021deep, nakada2023understanding} extensions. Most relevant to our context are sparse CCA methods, designed for a CCA failure case, where the number of samples is smaller than the number of dimensions. Thereby, they attempt to reduce the sample complexity of CCA by enabling the use of fewer paired samples \cite{witten2009penalized, cai2020sparse, wu2025nonlinear}. However, while these approaches reduce the reliance on large paired datasets, they fundamentally still require exact correspondences. In contrast, \ourmethod\ operates exclusively from unpaired data.

\paragraph{Incorporating Unpaired Data.} Driven by its availability, unpaired data has garnered substantial interest in recent years. However, most existing approaches leverage this data only in conjunction with auxiliary matching signals. For instance, semi-supervised methods utilize unpaired data to enhance models initially trained on paired samples \cite{nakada2023understanding, singh2022flava, devillers2024semi}. Weakly-supervised approaches significantly reduce the number of required pairs, relying predominantly on unpaired data but still necessitating a few exact pairs for guidance \cite{yacobi2025sue, roschmann2026sotalign, zhang2025sail, chen2023weakly, mo2023s, sahbi2018learning}. Other methods substitute explicit sample pairs with shared class labels \cite{timilsina2024identifiable, xi2024propensity}. Crucially, all of these paradigms depend on some form of complementary alignment signal and cannot function in a strictly unpaired regime as \ourmethod.

\paragraph{Fully Unpaired Methods.} A separate line of research tackles entirely unpaired data in specific applied domains, such as translation \cite{zhu2017unpaired, ma2019unpaired, zhang2024towards, chen2024contourdiff, wang2026ot}, segmentation \cite{yang2022toward, valindria2018multi, andrejic2025unpaired, de2024image}, and image restoration \cite{kartheek2025continual, lan2025exploiting}. However, these techniques are largely task-specific and do not explicitly extract a shared multiview representation space. Similarly, broader techniques like domain confusion \cite{ganin2015unsupervised, stark2020scim, zhang2022scdart} and domain adaptation \cite{hoshen2018unsupervised, jin2020minimum, lai2024empowering, choe2025universal} align marginal distributions but are not designed to extract maximally correlated components. In a different context, a related approach to analyze the similarity of Large Language Models embeddings in unpaired settings was recently presented in \cite{dar2025mini}. From a theoretical standpoint, \cite{alvarez2019towards} analyzed optimal transport under global invariances. The most closely related work to ours is  SCA \cite{timilsina2024identifiable}, which attempts to identify shared components across multiple views without pairs. However, their theoretical framework does not account for correlation, and as we demonstrate empirically, \ourmethod\ significantly better captures highly correlated components in practice.

\section{Background} \label{sec:background}
This section formalizes the basic definitions and mathematical notation utilized throughout the paper, primarily concerning CCA and the QAP.

\paragraph{Notations.} Let \(X, Y \in \mathbb{R}^{n \times d}\) represent whitened datasets consisting of \(n\) samples with \(d\) features. We denote \(\mathcal{P}\) as the set of \(n\times n\) permutation matrices, defined as \(\mathcal{P} = \{P\in\{0, 1\}^{n\times n}: P^T\mathbf{1} = P\mathbf{1} = \mathbf{1}\}\), where \(\mathbf{1}\) is the \(n\)-dimensional column vector consisting of all ones. For a matrix \(A \in \mathbb{R}^{n\times d}\), \(\sigma(A) = \left( \sigma_1(A), \sigma_2(A), \dots, \sigma_d(A) \right)\) represents the vector of all singular values, ordered in descending value such that \(\sigma_1(A) \geq \sigma_2(A) \geq \cdots \geq \sigma_d(A) \geq 0\) (assuming \(n > d\)). The Frobenius and nuclear norms are denoted by \(\lVert A \rVert_F = \sqrt{\tr(A^T A)}, \lVert A \rVert_*=\tr(\sqrt{A^T A})\), respectively. Finally, we define the Orthogonal group \(O(d) = \{Q \in \mathbb{R}^{d\times d}: Q^TQ=I\}\), and the Frobenius sphere \(S_F = \{S \in \mathbb{R}^{d\times d}: \lVert S \rVert_F = 1\}\). 

\paragraph{Canonical Correlation Analysis (CCA).} The total correlation (TC) of \(X\) and \(Y\) is defined as
\begin{equation} \label{eq:tc}
    \tc(X, Y) = \tr(X^T Y)
\end{equation}
CCA seeks the maximally correlated linear components of \(X\) and \(Y\). Namely, it returns a pair of projection matrices \(U^*, V^* \in \mathbb{R}^{d\times k}\) that maximize the total correlation (Eq. \eqref{eq:tc}):
\begin{equation}
    (U^*,V^*) = \cca(X,Y) = \arg\max_{U, V}\tc(XU, YV)\quad \text{s.t.}\quad U^T U = V^T V = I
\end{equation}
where \(k \ll d\), the number of components, is a tunable hyper-parameter.

\paragraph{Quadratic Assignment Problem (QAP).} Let \(A, B \in \mathbb{R}^{n\times n}\), and \(\Pi\) be the set of permutation functions (bijection) over \(\{1, 2, \dots, n\}\). The QAP seeks the optimal one-to-one assignment of \(n\) entities to \(n\) slots, maximizing the structural alignment between an interaction matrix \(B\) and a similarity matrix \(A\). Formally, the Koopmans-Beckmann QAP formulation \cite{koopmans1957assignment, loiola2007survey} is defined as:
\begin{equation} \label{eq:qap_kb}
    \pi^* = \qap(A, B) =\arg\max_{\pi\in\Pi} \sum_{i=1}^n \sum_{j=1}^n B_{ij}A_{\pi(i)\pi(j)}
\end{equation}
Using the set of permutation matrices, \(\mathcal{P}\), Eq. \eqref{eq:qap_kb} can be recast in matrix and TC notations:
\begin{equation}
    P^* = \qap(A, B) = \arg\max_{P\in\mathcal{P}} \tc(A^T, PB^TP^T)
\end{equation}

\section{Correlation of Unpaired Data} \label{sec:theory}

This section theoretically investigates the Unpaired CCA problem. For brevity, full derivations and proofs are deferred to App. \ref{app:proofs}. The analysis proceeds as follows: Sec. \ref{sec:mcp} formalizes a proxy pairing for the Unpaired CCA setting; Sec. \ref{sec:mcp_qap} proves an equivalence between a relaxation of this pairing and linear-kernel QAP; and Sec. \ref{sec:mcp_equiv} shows that these pairings coincide under mild assumptions. These results establish the equivalence between linear-kernel QAP and Unpaired CCA.

\subsection{Unpaired CCA Proxy Pairing} \label{sec:mcp}

\paragraph{Unpaired CCA.} In the unpaired scenario, we assume the rows (samples) of \(X\) and \(Y\) are related by an unknown ground-truth permutation \(P^*\). In practice, this assumption is quite mild, as it only requires a small set of representative anchors of each point set to correspond. We aim to find projections \(U'\) and \(V'\) under which the TC of the true pairing \(P^*\) matches that of the optimal paired projections \((U^*, V^*) = CCA(X, P^* Y)\):
\begin{equation} \label{eq:unpaired_cca_objective}
    \text{TC}(X U', P^* Y V') = \text{TC}(X U^*, P^* Y V^*)
\end{equation}

However, directly evaluating or optimizing Eq. \eqref{eq:unpaired_cca_objective} is impossible because it explicitly requires the ground-truth permutation \(P^*\), which is strictly unknown in the unpaired setting. To circumvent this, we assume that \(P^*\) maximizes the CCA objective. We then search for a proxy permutation \(P'\) that maximizes this criterion. Under this assumption, we can compute the corresponding projections via \((U', V') = \text{CCA}(X, P'Y)\), which consequently yields the desired solution from Eq. \eqref{eq:unpaired_cca_objective}. Crucially, this is a significantly weaker and more attainable requirement on the proxy pairing \(P'\) than demanding exact permutation recovery (\(P' \approx P^*\)). Still, the resulting optimization problem remains highly challenging, as the search space spans all possible permutations, together with two orthogonal projection matrices.

\begin{figure}
    \centering
    \includegraphics[width=0.86\linewidth]{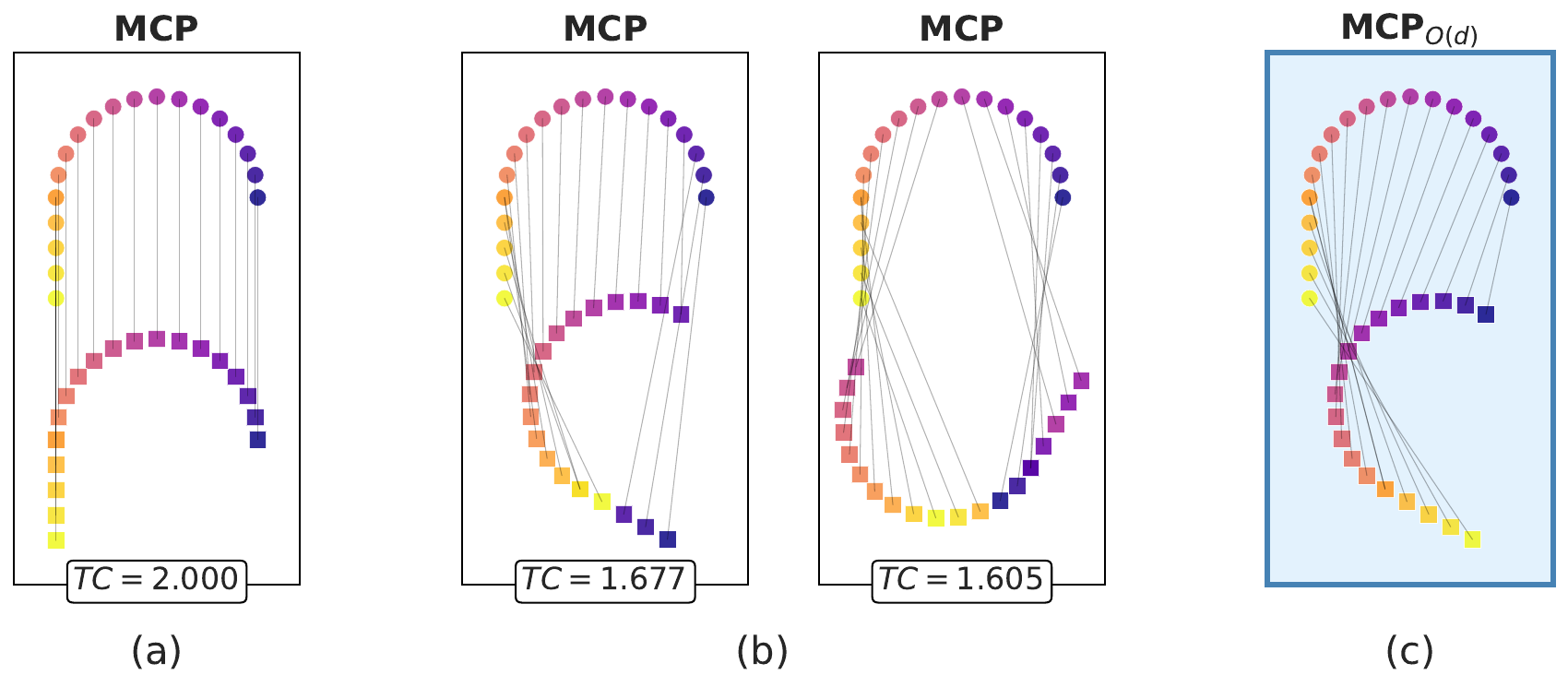}
    \caption{\textbf{MCP and \(\bf\textbf{MCP}_{O(d)}\) demonstration.} \textbf{(a)} MCP successfully reconstructs the true pairing on aligned data. \textbf{(b)} MCP depends on input rotation, which leads to poor performance on unaligned data. \textbf{(c)} Conversely, $\text{MCP}_{O(d)}$ optimizes over all orthogonal alignments, choosing the pairing with the highest correlation after the projection. In this example, it chooses the pairing from (a), enabling it to retrieve the true pairing, invariant to input (orthogonal) projections. A dynamic animation of this figure is available on our \href{https://shaham-lab.github.io/UCCA_page}{project page}.}
    \label{fig:mcp}
\end{figure}

\paragraph{Maximum Correlation Pairing.} An approach to simplify the search for the proxy pairing might be to bypass the projections entirely and select the permutation that directly maximizes the total correlation \cite{dowson1982frechet, knott1984optimal}. We formally define this as the Maximum Correlation Pairing (MCP):

\begin{equation}
    P' = \mcp(X, Y) = \arg\max_{P\in\mathcal{P}} \tc(X, PY)
\end{equation}

As depicted in Fig. \ref{fig:mcp}a, while standard MCP works well for aligned data (e.g., \(Y = X\)), it is highly sensitive to arbitrary rotations of the input space. Consequently, it fails to capture meaningful pairings when the data views are not pre-aligned (Fig. \ref{fig:mcp}b). Since CCA fundamentally seeks and is invariant to orthogonal transformations, the optimal pairing must be evaluated together with an appropriate orthogonal projection.

Thereby, we modify the problem by decoupling it into consecutive maximizations. We introduce two general sets of linear transformation functions, \(\mathcal{F}_1, \mathcal{F}_2\), defining the pairing \(\mcp_{\mathcal{F}_1, \mathcal{F}_2}(X, Y)\) as follows:
\begin{equation} \label{eq:mcp_f}
    P' = \mcp_{\mathcal{F}_1, \mathcal{F}_2}(X, Y) = \arg\max_{P\in\mathcal{P}}\max_{A\in\mathcal{F}_1, B\in\mathcal{F}_2} \tc(XA, PYB)
\end{equation}

When setting \(\mathcal{F}_1 = \mathcal{F}_2 = O(d)\), the inner maximization corresponds precisely to the CCA objective without dimensionality reduction. For brevity, w.l.o.g we fix \(\mathcal{F}_2 = O(d)\), and denote \(\mcp_{\mathcal{F}}(X,Y) = \mcp_{\mathcal{F}, O(d)}\). Fixing \(\mathcal{F}_1\) yields identical results. Unlike naive \(\mcp\), \(\mcp_{O(d)}\) (Fig. \ref{fig:mcp}c) is robust to orthogonal projections, successfully resolving the alignment failure illustrated in Fig. \ref{fig:mcp}b. For brevity, we assume a unique maximizer for \eqref{eq:mcp_f}. Our results remain valid without this assumption. The derivations proceed identically by choosing a representative element (see App. \ref{app:proofs}).

Consequently, \(\mcp_{O(d)}\) serves as a computable proxy \(P'\) that can be evaluated entirely without paired samples. We argue that this proxy robustly recovers the TC under the true underlying permutation, a claim we substantiate through comprehensive empirical evaluations in Sec. \ref{sec:results}.

\begin{takeaway}{1}
    \(\mcp_{O(d)}\) defines a proxy pairing for Unpaired CCA
\end{takeaway}

\subsection{Connection to QAP} \label{sec:mcp_qap}

In practice, optimizing over either discrete permutations or orthogonal projections individually is already a difficult problem \cite{burkard1984quadratic, edelman1998geometry}. Consequently, optimizing over both simultaneously makes a direct computation of \(\mcp_{O(d)}\) computationally prohibitive (see App. \ref{app:direct_optimization}). Therefore, we relax the optimization space from \(O(d)\) to a broader set: the (scaled) Frobenius sphere, \(\sqrt{d} S_F\). As we will demonstrate, under a mild assumption, this relaxation does not alter the optimal pairing while making the problem computationally feasible.

Interestingly, as formalized in Thm. \ref{thm:mcp_qap}, \(\text{MCP}_{S_F}\) reduces to a special case of the well-studied QAP. While QAP is generally NP-hard, efficient approximate solvers exist \cite{vogelstein2015fast, heider1972computationally, loiola2007survey}, making \(\text{MCP}_{S_F}\) a practically attainable pairing.

\begin{theorem} \label{thm:mcp_qap}
    \(\mcp_{\sqrt{d}S_F}(X, Y) = \mcp_{S_F}(X, Y) = \qap(XX^T, YY^T)\).
\end{theorem}

\begin{takeaway}{2}
   \(\text{MCP}_{S_F}\) is practically attainable via QAP
\end{takeaway}

\begin{figure}
    \centering
    \includegraphics[width=\linewidth]{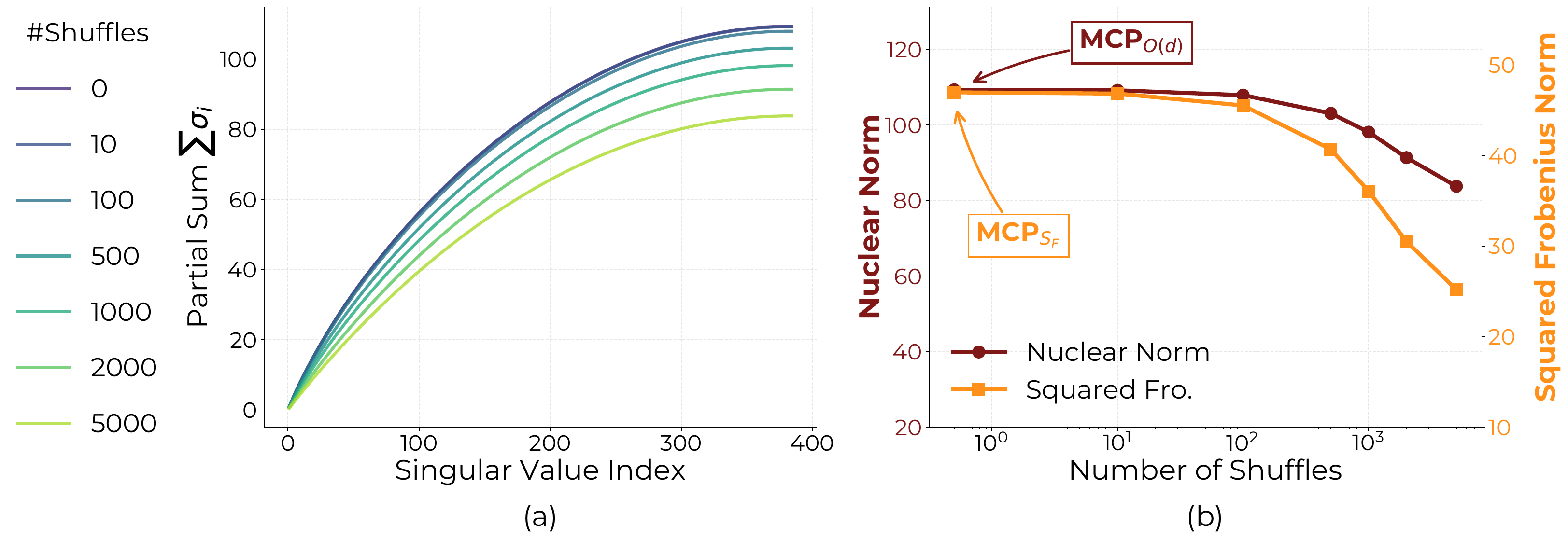}
    \vspace{-15pt}
    \caption{\textbf{(a) Validation of Assum. \ref{assum:weak_majorization}:} Partial cumulative sums of singular values for the true permutation \(P^*\) (\#shuffles = 0) and various random shuffles on the Flickr dataset (8,000 samples; see Sec. \ref{sec:results}). As the permutation diverges from \(P^*\), the partial sums strictly decrease, confirming the majorization property. \textbf{(b) Validation of Thm. \ref{thm:mcp_equiv}.} A comparison of \(\mcp_{O(d)}\) (Nuclear) and \(\mcp_{S_F}\) (Frobenius) objectives using the setup from (a). Both objectives decay as the permutation moves away from \(P^*\), and both are locally maximized by the same ground-truth permutation.}
    \label{fig:theory}
\end{figure}

\subsection{Equivalence of \(\bf\textbf{MCP}_{O(d)}\) and \(\bf\textbf{MCP}_{S_F}\)} \label{sec:mcp_equiv}

Thus far, we have defined \(\text{MCP}_{O(d)}\), and established that \(\text{MCP}_{S_F}\) is an attainable pairing via QAP approximations. A critical question remains: how do these two pairings relate? To bridge this gap, we first observe that both \(\text{MCP}_{O(d)}\) and \(\text{MCP}_{S_F}\) can be reformulated as optimization problems exclusively over \(\mathcal{P}\):
\begin{proposition} \label{prop:mcpf_norms}
    \[\mcp_{O(d)}(X, Y) = \arg\max_{P\in\mathcal{P}} \lVert X^T P Y \rVert_*,
    \quad
    \mcp_{S_F}(X, Y) = \arg\max_{P\in\mathcal{P}} \lVert X^T P Y \rVert_F\]
\end{proposition}

Let \(M_P = X^T P Y\) denote the cross-correlation matrix of \(X\) and \(Y\) under permutation \(P\). Prop. \ref{prop:mcpf_norms} translates into an optimization over \(\sigma(M_P)\), the singular values of \(M_P\):

\begin{equation} \label{eq:mcpf_svals}
    \mcp_{O(d)}(X, Y) = \arg\max_{P\in\mathcal{P}} \sum_{i=1}^d \sigma_i(M_P),
    \quad
    \mcp_{S_F}(X, Y) = \arg\max_{P\in\mathcal{P}} \sum_{i=1}^d \sigma^2_i(M_P)
\end{equation}

Denote \(P_O = \mcp_{O(d)}\) and \(P \ne \mcp_{O(d)}\). By Eq. \eqref{eq:mcpf_svals}, the sum of singular values under \(P_O\) is strictly maximal:
\begin{equation}
    \sum_{i=1}^d \sigma_i(M_{P_O}) > \sum_{i=1}^d\sigma_i(M_{P})
\end{equation}

In Assumption \ref{assum:weak_majorization} below, we assume this inequality holds not just for the total sum, but is preserved across all partial sums.

\begin{assumption} \label{assum:weak_majorization}
    Let \(P_O = \mcp_{O(d)}\) and \(P \neq \mcp_{O(d)}\), then \(\sigma(M_{P_O})\) weakly majorizes \(\sigma(M_{P})\). Namely,
    \[\sum_{i=1}^m \sigma_i(M_{P_O}) \geq \sum_{i=1}^m \sigma_i(M_{P})\quad \forall m \in \{1, 2, \dots, d\}\]
\end{assumption}

This assumption is relatively mild. As demonstrated in App. \ref{app:proofs}, Prop. \ref{prop:finite_assum_support} establishes that Assum. \ref{assum:weak_majorization} is satisfied when \(Y\) is an isometry of \(X\) (i.e., \(Y = XQ\), where \(Q \in O(d)\)). Furthermore, Prop. \ref{prop:infinite_assum_support} proves that as \(n\to\infty\), the assumption holds with probability $1$ for any random permutation. We empirically corroborate this weak-majorization behavior on real datasets in Fig. \ref{fig:theory}a and App. \ref{app:assum_valid}. Building upon this, we derive the following theorem.

\begin{theorem} \label{thm:mcp_equiv}
    Under Assumption \ref{assum:weak_majorization}, \(\text{MCP}_{O(d)}(X, Y) = \text{MCP}_{S_F}(X, Y)\)
\end{theorem}

Fig. \ref{fig:theory}b empirically validates Thm. \ref{thm:mcp_equiv} on the Flickr dataset, with further support provided by the evaluations in Sec. \ref{sec:results} and App. \ref{app:thm_valid}.

\begin{takeaway}{3}
    \centering
    Under realistic data assumptions, \(\text{MCP}_{O(d)} = \text{MCP}_{S_F}\)
\end{takeaway}

Theorem \ref{thm:mcp_equiv} forms the theoretical cornerstone of our method. Synthesizing these findings, we conclude that under realistic assumptions, the computationally attainable \(\text{MCP}_{S_F}\) perfectly aligns with the \(\text{MCP}_{O(d)}\) pairing. Ultimately, this mathematically justifies the core premise illustrated in Fig. \ref{fig:open}: \textit{Unpaired CCA can be solved with linear kernel QAP}.

\section{Unpaired CCA} \label{sec:method}

\begin{figure}[htbp]
  \centering
  \begin{minipage}{0.48\textwidth}
    \centering
    \includegraphics[width=\linewidth]{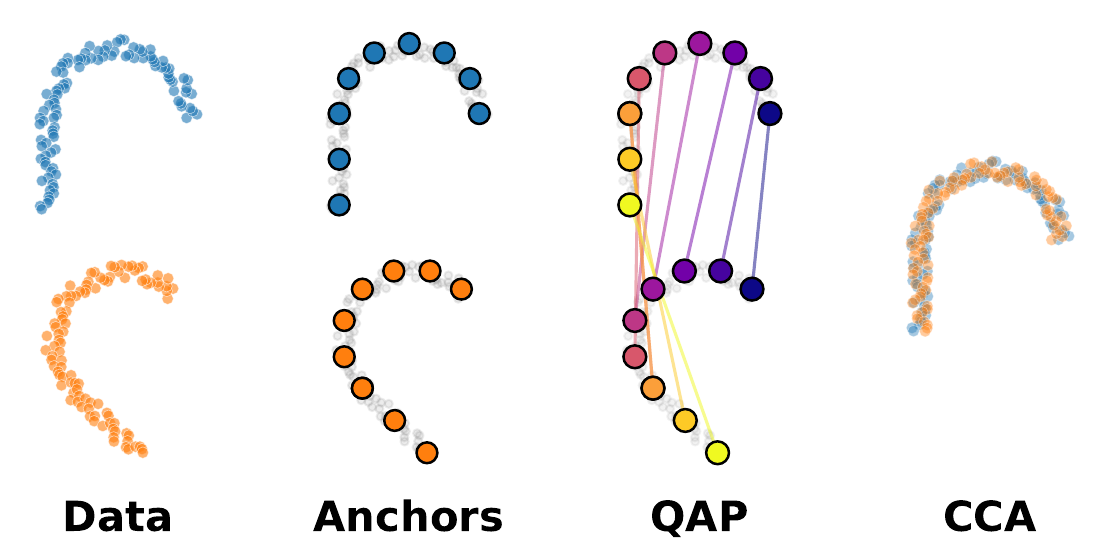}
    \vspace{-15pt}
    \caption{Unpaired CCA method overview.}
    \label{fig:method}
  \end{minipage}
  \hfill
  \begin{minipage}{0.48\textwidth}
    \begin{algorithm}[H]
      \caption{Unpaired CCA}
      \label{alg:method}
      \begin{algorithmic}[1]
        \Require Unpaired data \(X, Y\)
        \Ensure Projectors \(U, V\)
        \State \(Z_X, Z_Y \gets \text{Whiten}(X, Y)\)
        \State \(A_X, A_Y \gets \text{Anchors}(Z_X, Z_Y)\)
        \State \(P' \gets \text{QAP}(A_X, A_Y)\)
        \State \(U, V \gets \text{CCA}(A_X, P'A_Y)\)
        \State \Return \(U, V\)
      \end{algorithmic}
    \end{algorithm}
  \end{minipage}
\end{figure}

Building directly upon the theoretical foundations established in Sec. \ref{sec:theory}, we derive a practical algorithm for maximizing correlation in the absence of paired data, which we term Unpaired CCA (\ourmethod). This section outlines the general algorithmic framework of \ourmethod, deferring specific architectural details to App. \ref{app:implementation_details}. Notably, despite its simplicity, this approach significantly outperforms existing baselines in capturing cross-view correlations in the unpaired data scenario (Sec. \ref{sec:results}).

As concluded in Sec. \ref{sec:theory}, the \(\mcp_{O(d)}\) pairing can be attained by solving a linear-kernel QAP. However, the QAP is NP-hard. While efficient approximate solvers exist, applying them directly to the entire dataset becomes computationally prohibitive for a large number of samples, \(n\). To ensure scalability in practice, \ourmethod\ adopts an anchor-based approach \cite{williams2000using, de2004sparse}, operating by first selecting a subset of \(k\) representative anchors from each modality. A standard and highly practical choice of anchors is the K-Means centroids \cite{kumar2012sampling, coates2011analysis}. We apply K-Means clustering within each modality and extract the resulting centroids to serve as our anchors.

Once extracted, these \(k\) anchors are matched using the linear-kernel QAP formulation detailed in Thm. \ref{thm:mcp_qap}, yielding an estimated permutation matrix \(P'\). For robustness, this process is repeated several times, and the matched anchors are concatenated. Finally, we treat these matched anchors as pseudo-pairs and apply standard CCA to compute the final projection matrices. Importantly, \ourmethod\ does not assume any existing correspondences in the data samples. The complete procedure is summarized in Alg. \ref{alg:method} and visualized in Fig. \ref{fig:method}.

\section{Experimental Results} \label{sec:results}

This section details the empirical evaluation of \ourmethod. Our experimental design aims to validate \ourmethod's capacity to recover latent cross-modal correlations, and empirically substantiate our theoretical framework. Comprehensive implementation details, including architectural configurations and hyperparameter selections, are documented in App. \ref{app:implementation_details}.

\paragraph{Datasets.} To rigorously assess \ourmethod's ability to recover the true underlying correlation, we evaluate it on benchmark datasets where ground-truth pairings are known but strictly withheld during training. Specifically, we utilize four diverse real-world datasets: SNARE \cite{demetci2022scot} (single-cell multi-omics), Flickr \cite{hodosh2013framing}, COCO \cite{lin2014microsoft} (image-text), and Handwritten digits benchmark \cite{van1998handwritten, asuncion2007uci}. For the Handwritten dataset, we employ three distinct views: profile correlations (PC), Karhunen-Loève coefficients (KL), and pixel averages (PA). Across these benchmarks, we evaluate \ourmethod\ on a total of six distinct dataset configurations. Crucially, we enforce a strictly unpaired training split by dividing the training data into two disjoint subsets, allocating one half entirely to the first view and the remainder to the second, thereby eliminating any paired correspondences during training.

\paragraph{Baselines.} As \ourmethod\ represents the first methodology explicitly designed to optimize correlation exclusively from unpaired samples, direct baselines do not currently exist. Therefore, we compare \ourmethod\ against recent unpaired alignment techniques. The most structurally similar approaches are SCA \cite{timilsina2024identifiable}, which extracts shared linear components from unpaired multiview data, and UCA \cite{hoshen2018unsupervised}, which learns shared representations motivated by correlation objectives. We benchmark against the linear formulation of UCA, deferring results for its non-linear variant to App. \ref{app:uca_nonlinear}. Alternative methods, such as J-MDS \cite{chenunsupervised}, SCOTv1 \cite{demetci2022scot}, and SCOTv2 \cite{demetci2022scotv2}, extract shared representations without establishing a functional mapping, rendering them inapplicable to the generalizable setting. To enable comparison to these, we adapt them by fitting a standard CCA model to the pseudo-pairs they infer during training (similar to \ourmethod's approach), denoting these extensions as -CCA variants.

\begin{figure}[t]
    \centering
    \includegraphics[width=\linewidth]{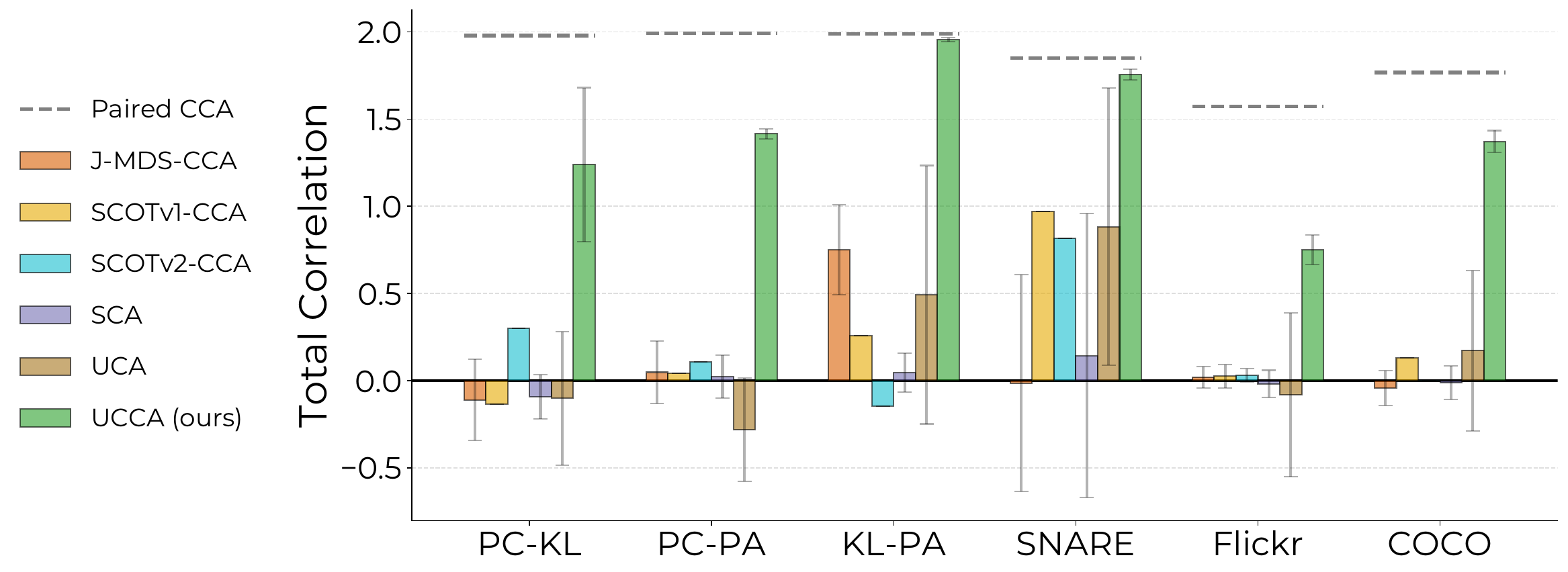}
    \vspace{-15pt}
    \caption{\textbf{Total Correlation results.} Total correlation achieved across six distinct dataset configurations. \ourmethod\ consistently extracts highly correlated components, significantly outperforming all baselines, and often approaching the paired CCA upper bound.}
    \label{fig:tc}
\end{figure}

\paragraph{Total Correlation.} Fig. \ref{fig:tc} evaluates the fundamental objective of an unpaired CCA method: its ability to successfully capture the true correlation of the hidden ground-truth pairings. We compute the total correlation as defined in Eq. \eqref{eq:tc}, \(\tc(X, P^*Y) = \tr(X^TP^*Y)\), after orthogonalizing \(X\) and \(Y\) separately to remove correlation repetition. The figure reports the total correlation captured by the top two learned components. Evidently, \ourmethod\ is the sole approach that consistently recovers the underlying correlation. It significantly outperforms all baselines and often approaches the paired CCA upper bound.

\begin{wrapfigure}{r}{0.4\textwidth} 
    \vspace{-5pt}
    \centering
    \includegraphics[width=0.4\textwidth]{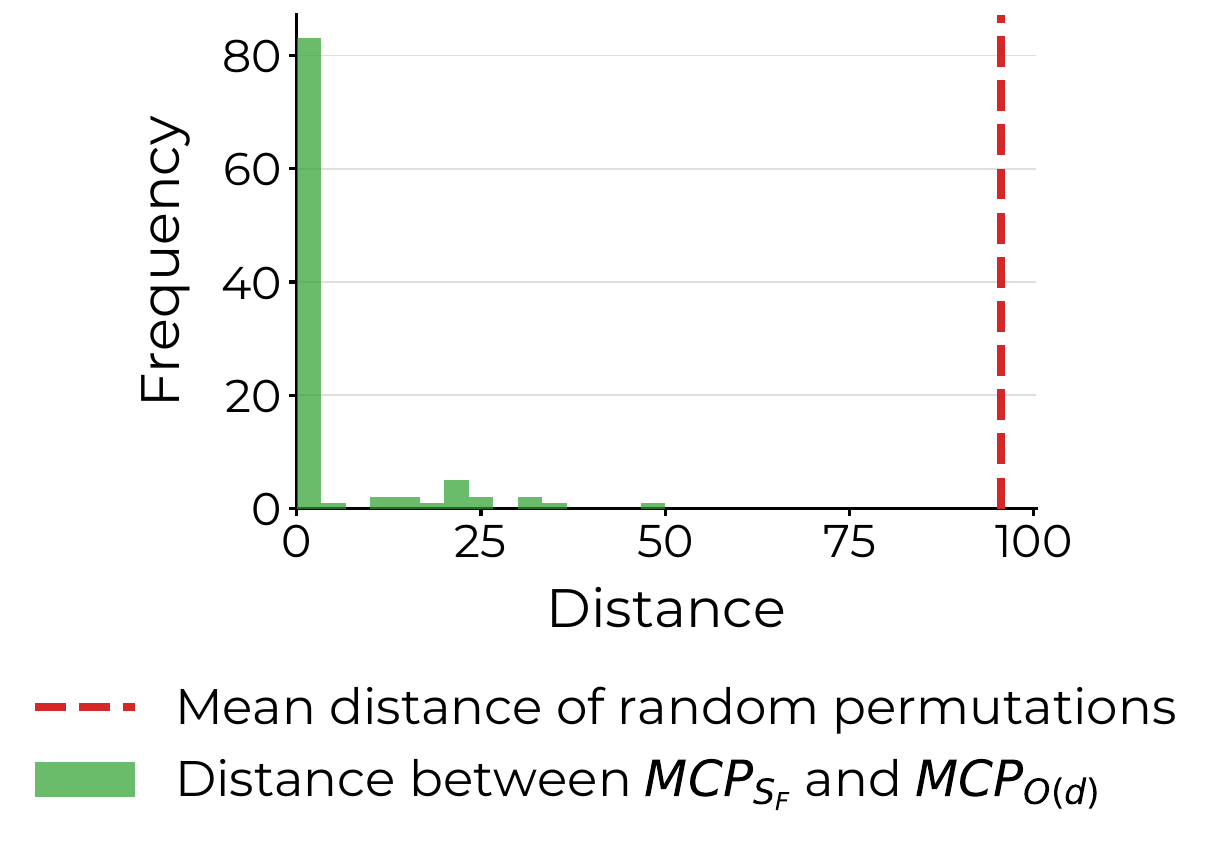}
    \vspace{-18pt}
    \caption{\textbf{\(\text{MCP}_{O(d)} = \text{MCP}_{S_F}\) in practice.} Kendall Tau Distances between \(\mcp_{S_F}\) and approximated \(\mcp_{O(d)}\), computed on the anchors of the Handwritten benchmark (PC-PA) across 100 initializations. The strict concentration at zero empirically validates Thm. \ref{thm:mcp_equiv}.}
    \label{fig:mcp_equiv}
    \vspace{-25pt}
\end{wrapfigure}

The comparison with J-MDS-CCA, SCOTv1-CCA and SCOTv2-CCA is of particular interest, as it isolates the core effect of optimizing the \(\mcp_{O(d)}\) pairing. It clearly demonstrates that the \(\mcp_{O(d)}\) objective provides a superior pairing strategy for CCA scenarios when compared to alternative approaches.

\begin{table}[t]
\centering
\caption{\textbf{Cross-view classification results.} Mean accuracy (\(\pm\) std) of a kNN classifier fitted on the embeddings of one view and evaluated on the other. Bold and underlined values indicate the best and second-best performances, respectively. \ourmethod\ significantly outperforms all baselines across four datasets while maintaining comparable performance on the fifth.}
\label{tab:classification}

\begin{adjustbox}{width=\textwidth}
\begin{tabular}{lccccc}
\toprule
Method & PC-KL & PC-PA & KL-PA & SNARE & COCO \\
\midrule
J-MDS-CCA & \( 0.083 \pm 0.021 \) & \( 0.087 \pm 0.021 \) & \( \underline{0.211} \pm 0.033 \) & \( 0.247 \pm 0.224 \) & \( \mathbf{0.237} \pm 0.079 \) \\
SCOTv1-CCA & \( 0.074 \pm 0.000 \) & \( \underline{0.121} \pm 0.000 \) & \( 0.102 \pm 0.000 \) & \( \underline{0.650} \pm 0.000 \) & \( 0.141 \pm 0.000 \) \\
SCOTv2-CCA & \( \underline{0.148} \pm 0.000 \) & \( 0.090 \pm 0.000 \) & \( 0.113 \pm 0.000 \) & \( 0.646 \pm 0.000 \) & \( 0.164 \pm 0.000 \) \\
SCA & \( 0.097 \pm 0.014 \) & \( 0.083 \pm 0.011 \) & \( 0.100 \pm 0.022 \) & \( 0.312 \pm 0.255 \) & \( 0.161 \pm 0.020 \) \\
UCA & \( 0.088 \pm 0.016 \) & \( 0.100 \pm 0.017 \) & \( 0.159 \pm 0.090 \) & \( 0.378 \pm 0.297 \) & \( 0.132 \pm 0.058 \) \\
\rowcolor{lightgreen}
UCCA (ours) & \( \mathbf{0.292} \pm 0.074 \) & \( \mathbf{0.315} \pm 0.025 \) & \( \mathbf{0.564} \pm 0.022 \) & \( \mathbf{0.849} \pm 0.021 \) & \( \underline{0.232} \pm 0.016 \) \\
\midrule
\textcolor{gray}{Paired CCA} & \textcolor{gray}{\( 0.562 \pm 0.000 \)} & \textcolor{gray}{\( 0.516 \pm 0.000 \)} & \textcolor{gray}{\( 0.629 \pm 0.000 \)} & \textcolor{gray}{\( 0.887 \pm 0.000 \)} & \textcolor{gray}{\( 0.232 \pm 0.009 \)} \\
\bottomrule
\end{tabular}
\end{adjustbox}
\end{table}

\paragraph{Cross-view Classification.} To evaluate the downstream performance of \ourmethod, we apply it to a cross-view classification task (Tab. \ref{tab:classification}). Using the shared representations learned exclusively from unpaired data, we train a kNN classifier on one view and evaluate its accuracy on the other. As shown, \ourmethod\ outperforms all baselines across four datasets while maintaining comparable performance on the fifth, often approaching the performance of the paired CCA. The Flickr dataset is omitted from this experiment as it lacks class labels.

\paragraph{Validating \(\bf\textbf{MCP}_{O(d)} = \textbf{MCP}_{S_F}\).} Fig. \ref{fig:mcp_equiv} empirically validates the equivalence between \(\text{MCP}_{O(d)}\) and \(\text{MCP}_{S_F}\). Specifically, it visualizes the distance between \(\mcp_{S_F}\), retrieved via an approximate QAP solver, and an estimated \(\mcp_{O(d)}\), denoted as \(\tilde{P_O}\). Conceptually, \(\tilde{P_O}\) represents the projection of \(\mcp_{S_F}\) onto the \(\mcp_{O(d)}\) solution space. If \(\tilde{P_O} = \mcp_{S_F}\), the approximation is exact, establishing that \(\mcp_{O(d)} = \mcp_{S_F}\). To assess this, we computed both pairings across all real datasets using different anchor initializations. We quantified the divergence between the resulting permutations using the Kendall Tau distance, which counts the number of pairwise inversions. As shown, these distances are heavily concentrated at zero, providing strong empirical validation that Thm. \ref{thm:mcp_equiv} holds in practice (see App. \ref{app:valid_thm2_exp} for further details and results).

\begin{wrapfigure}{r}{0.5\textwidth} 
    \vspace{-10pt}
    \centering
    \includegraphics[width=0.5\textwidth]{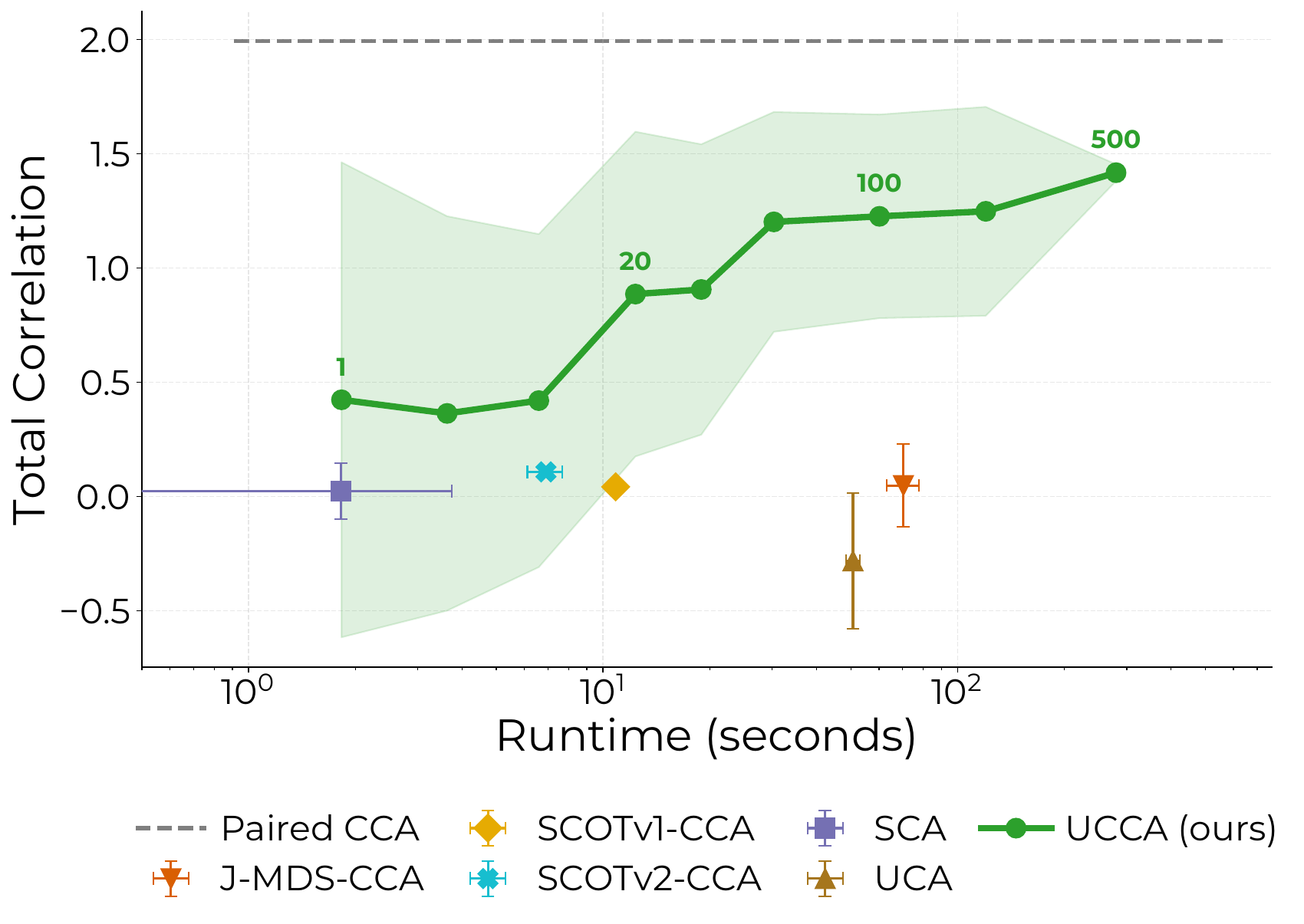}
    \vspace{-15pt}
    \caption{\textbf{Runtime and Performance.} Total correlation versus runtime for \ourmethod\ and baselines on the handwritten (PC-PA) benchmark. \ourmethod\ is improved in robustness and accuracy with more clustering repetitions, yet outperforms all baselines even with fewer repetitions and shorter runtimes.}
    \label{fig:runtime}
    \vspace{-20pt}
\end{wrapfigure}

\paragraph{Computational Efficiency.} Relying on QAP, an NP-hard problem, naturally introduces concerns regarding runtime bottlenecks. However, as detailed in Fig. \ref{fig:runtime} (and App. \ref{app:runtime}), \ourmethod\ demonstrates a highly favorable trade-off between runtime and performance. It consistently outperforms baseline methods even when configured with fewer clustering repetitions for faster execution, and requires a total runtime of only several minutes for each dataset. The algorithmic reliance on a smaller subset of representative anchors, combined with efficient approximate solvers, effectively mitigates this computational overhead and ensures practical scalability. We further discuss the computational complexity in App. \ref{app:complexity}.

\begin{figure}
    \centering
    \includegraphics[width=0.98\linewidth]{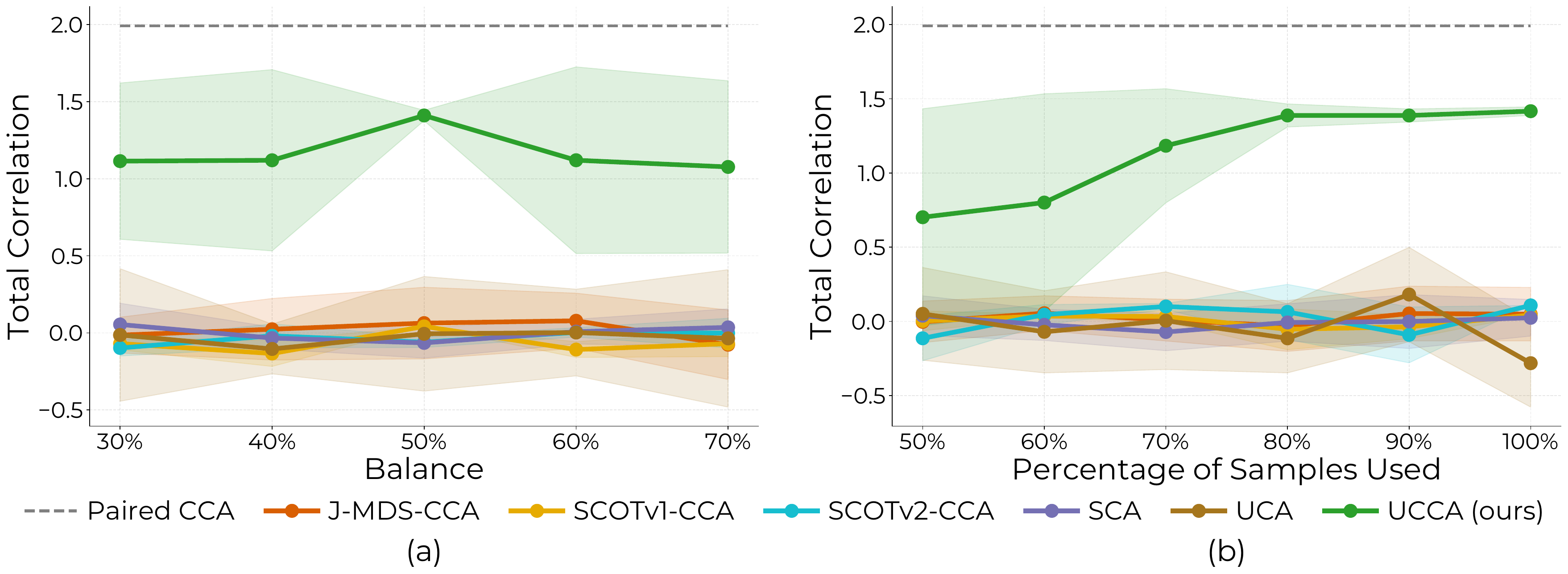}
    \vspace{-3pt}
    \caption{Total correlation performance on the Handwritten benchmark (PC-PA) across \textbf{(a) data imbalances} and \textbf{(b) varying amounts of unpaired data}. \ourmethod\ consistently outperforms baselines, achieving its peak performance when the maximum amount of data is accessible. This demonstrates that abundant unpaired data alone provides a strong alignment signal, improving performance without requiring any paired correspondences.}
    \label{fig:ablation}
\end{figure}

\paragraph{Ablations.} We further evaluate the stability of \ourmethod\ under varying data conditions using the PC-PA views of the Handwritten benchmark (Fig. \ref{fig:ablation}). Fig. \ref{fig:ablation}a illustrates the \textbf{effect of unbalanced data}, demonstrating that \ourmethod\ is robust to view imbalances and performs optimally when all data is accessible during training. Fig. \ref{fig:ablation}b examines the \textbf{effect of unpaired data amount}, showing that the consistency and performance of \ourmethod\ increase alongside the volume of available unpaired data. Together, these results demonstrate that accessible unpaired data alone provides a strong alignment signal. Additional ablations validating the choice  of K-Means centroids as anchors and the number of anchors used are provided in App. \ref{app:anchors}.

\section{Conclusions and Summary} \label{sec:conclusions}
In this paper, we introduced and formalized the problem of Unpaired Canonical Correlation Analysis. Through a rigorous theoretical investigation, we defined a proxy pairing for the unpaired scenario (\(\mcp_{O(d)}\)), and identified a practically attainable relaxation via the Quadratic Assignment Problem (\(\mcp_{S_F}\)). Crucially, we proved that under realistic data assumptions, these two pairings coincide. Grounded in these theoretical insights, we developed \ourmethod, the first practical algorithm capable of extracting canonical correlation components relying exclusively on unpaired data. Our empirical evaluations across multiple real-world datasets demonstrate that \ourmethod\ significantly outperforms existing unpaired alignment baselines in recovering the true underlying correlation components.

\paragraph{Limitations.} While \ourmethod\ represents an advancement over prior unpaired baselines, a performance gap naturally remains relative to the theoretical upper bound achieved by the paired CCA. Additionally, our current framework focuses on the bi-view setting. Extending this formulation to multi-view scenarios is a compelling direction for future research, particularly because the difficulty of acquiring cross-modal correspondences scales significantly with the number of views.

Ultimately, this work bridges a critical gap between classical statistical multiview learning and the rapidly expanding paradigm of unpaired data integration. By demonstrating that latent correlations can be systematically recovered without explicit correspondence, we open new avenues for applying principled statistical techniques to massive, unstructured datasets in domains where paired samples are prohibitively expensive or impossible to acquire.

\bibliographystyle{plain} 

\bibliography{bib}

\newpage
\appendix

\section{Full Proofs} \label{app:proofs}
This section details the full theoretical derivations and proofs.

Recall the notations:
\(X, Y \in \mathbb{R}^{n \times d}\) represent whitened datasets consisting of \(n\) samples with \(d\) features. That is, \(X\mathbf{1}_n = \mathbf{0}_d, X^TX = I_d\), and \(Y^T\mathbf{1}_n = \mathbf{0}_d, Y^TY = I_d\), where \(\mathbf{1}_n\) and \(\mathbf{0}_d\) are the \(n\)-dimensional and \(d\)-dimensional column vectors consisting of all ones and zeros, respectively.

We denote \(\mathcal{P}\) as the set of \(n\times n\) permutation matrices, defined as \[\mathcal{P} = \{P\in\{0, 1\}^{n\times n}: P^T\mathbf{1}_n = P\mathbf{1}_n = \mathbf{1}_n\}.\]
For a matrix \(A \in \mathbb{R}^{n\times d}\), \(\sigma(A) = \left( \sigma_1(A), \sigma_2(A), \dots, \sigma_d(A) \right)\) represents the vector of all singular values, ordered in descending value such that \(\sigma_1(A) \geq \sigma_2(A) \geq \cdots \geq \sigma_d(A) \geq 0\) (assuming \(n > d\)).
The Frobenius and nuclear norms are denoted by \(\lVert A \rVert_F = \sqrt{\tr(A^T A)}, \lVert A \rVert_*=\tr(\sqrt{A^T A})\), respectively.
We define the Orthogonal group
\[O(d) = \{Q \in \mathbb{R}^{d\times d}: Q^TQ = I\},\]
and the Frobenius sphere \[S_F = \{S \in \mathbb{R}^{d\times d}: \lVert S \rVert_F = 1\}.\] 

\subsection{Unpaired CCA Proxy Pairing}

\begin{definition}\label{def:tc}
    Given two whitened sets of samples \(X, Y \in \mathbb{R}^{n\times d}\), their \underline{total correlation} is
    \[\text{TC}(X, Y) = \tr(X^TY)\]
\end{definition}

We next define the two pairing families discussed in the paper. These are sets of pairings, consisting of all the maximizers of the objectives.

\begin{definition} \label{def:mcp}
    The \underline{Maximum Correlation Pairing} (MCP) is
    \[\mathcal{P}' = \text{MCP}(X, Y) = \arg\max_{P\in\mathcal{P}}\text{TC}(X, PY)\]
\end{definition}

\begin{definition} \label{def:mcpf}
    Given two linear functions classes \(\mathcal{F}_1, \mathcal{F}_2 \subseteq \mathbb{R}^{d\times d}\),
    \[\mathcal{P}' = \text{MCP}_{\mathcal{F}_1, \mathcal{F}_2}(X, Y) = \arg\max_{P \in \mathcal{P}} \max_{A \in \mathcal{F}_1, B \in \mathcal{F}_2} \text{TC}(XA, PYB)\]
\end{definition}

As established in the main text, our objective requires evaluating pairings under valid CCA projections, necessitating optimization over the orthogonal group, \(O(d)\). For brevity, we fix \(\mathcal{F}_2 = O(d)\), and denote in short \(\mcp_{\mathcal{F}}(X,Y) = \mcp_{\mathcal{F}, O(d)}\). While the resulting pairing, \(\mcp_{O(d)}\), naturally aligns with our correlation maximization goals, directly optimizing over orthogonal matrices alongside permutations is computationally prohibitive. To render the problem tractable, we relax the function class from the orthogonal group to its superset, the scaled Frobenius sphere, \(\sqrt{d}S_{F}\). This relaxation yields \(\mcp_{S_{F}}\), which we next show is a computationally attainable proxy pairing.

\subsection{Connection to QAP}

A crucial property of both \(\mcp_{O(d)}\) and \(\mcp_{S_{F}}\) is that they can be equivalently reformulated as maximization problems exclusively over the permuted cross-correlation matrix, \(M_{P} = X^{T}PY\). This reformulation is highly advantageous: it implicitly absorbs the projection function class into the objective's norm, entirely circumventing the need to explicitly parameterize or compute the continuous transformations. We formalize these equivalences in Prop. \ref{prop:od_nuclear} and \ref{prop:sf_frob}.

\begin{proposition} \label{prop:od_nuclear}
    \[\text{MCP}_{O(d)} = \arg\max_{P\in\mathcal{P}} \lVert X^T P Y \rVert_*\]
    where \(\lVert \cdot \rVert_*\) is the Nuclear norm.
\end{proposition}

\begin{proof}
    \[\text{MCP}_{O(d)} = \arg\max_{P \in \mathcal{P},\ U, V \in \mathcal{O}(d)} \text{TC}(XU, PYV) = \arg\max_{P\in\mathcal{P}}\max_{U, V\in O(d)} \tr(U^T X^T P Y V)\]
    \[\overset{(*)}{=} \arg\max_{P\in\mathcal{P}} \lVert X^T P Y \rVert_* = \arg\max_{P\in\mathcal{P}}\sum_{i=1}^d \sigma_i(M_P)\]

    To prove \((*)\), we'll use Von Neumann's trace inequality. Denote \(Q = UV^T\). Notice that
    \[\tr(U^T X^T P Y V) = \tr(V U^T X^T P Y) = \tr(Q^T X^T P Y)\]
    For any \(Q \in O(d)\)
    \[\tr(Q^T X^T P Y) = \tr(Q^T M_P) \leq \sum_{i=1}^d \sigma_i(Q)\sigma_i(M_P)\]
    As \(Q\) is orthogonal, \(\sigma_i(Q) = 1\) for all \(i\). Hence
    \[\tr(Q^T M_P) \leq \sum_{i=1}^d \sigma_i(M_P) = \lVert M_P \rVert_* = \lVert X^T P Y \rVert_*\]

    As for achieving the maximum, consider \(M_P = U_M\Sigma V_M^T\) the SVD of \(M_P\). Then, \(Q^* = U_MV_M^T\) satisfies
    \[\tr({Q^*}^T M_P) = \tr((U_MV_M^T)^T U_M\Sigma V_M^T) = \tr(V_MU_M^T U_M \Sigma V_M^T)\]
    \[= \tr(V_M\Sigma V_M^T) = \tr(V_M^T V_M\Sigma) = \tr(\Sigma) = \lVert M_P\rVert_* = \lVert X^T P Y\rVert_*\]
\end{proof}

\begin{proposition} \label{prop:sf_frob}
    For all \(c > 0\),
    \[\text{MCP}_{c\cdot S_F} = \arg\max_{P\in\mathcal{P}} \lVert X^T P Y \rVert_F\]
    where \(\lVert \cdot \rVert_F\) is the Frobenius norm.
\end{proposition}

\begin{proof}
    \[\text{MCP}_{c\cdot S_F} = \arg\max_{P\in\mathcal{P}}\max_{S\in c\cdot S_F, U \in O(d)} \text{TC}(XS, PYU) = \arg\max_{P\in\mathcal{P}}\max_{S\in c\cdot S_F, U \in O(d)} \tr(S^T X^T P Y U)\]
    \[ = \arg\max_{P\in\mathcal{P}}\max_{S\in c\cdot S_F, U \in O(d)} \langle S, X^T P Y U\rangle_F\]
    By the Cauchy-Schwarz inequality, this is
    \[= \arg\max_{P\in\mathcal{P}, U \in O(d)}  c\lVert X^T P Y U \rVert_F = \arg\max_{P\in\mathcal{P}}  c\lVert X^T P Y \rVert_F\]
    \[= \arg\max_{P\in\mathcal{P}}  \lVert X^T P Y \rVert_F = \arg\max_{P\in\mathcal{P}}\sum_{i=1}^d \sigma_i^2(M_P)\]
\end{proof}

We can now establish our first main theoretical result: the exact equivalence between \(\mcp_{S_F}\) and the QAP, formalized in Thm. \ref{thm:mcp_qap}.

\newcounter{oldTheorem}
\setcounter{oldTheorem}{\value{theorem}}
\setcounter{theorem}{0}

\begin{theorem}
    Denoting QAP the quadratic assignment problem,
    \[\text{MCP}_{\sqrt{d} S_F}(X, Y) = \text{MCP}_{S_F}(X, Y) = QAP(XX^T, YY^T)\]
\end{theorem}

\setcounter{theorem}{\value{oldTheorem}}

\begin{proof}
    For the left equality, by Prop. \ref{prop:sf_frob}, for any \(c > 0\),
    \[\text{MCP}_{c\cdot S_F} = \text{MCP}_{S_F}\]
    This is true in particular for \(c = \sqrt{d}\).

    We now establish the formal equivalence between the relaxed objective \(\mcp_{S_F}\) and the well-studied QAP. Recognizing that \(\lVert X^T P Y \rVert_F \geq 0\) for any permutation \(P\), we can maximize its squared value:
    \[\text{MCP}_{S_F} = \arg\max_{P\in\mathcal{P}}  \lVert X^T P Y \rVert_F^2 = \arg\max_{P\in\mathcal{P}}  \tr((X^T P Y)^T(X^T P Y))\]
    \[= \arg\max_{P\in\mathcal{P}}  \tr(Y^TP^TXX^T P Y)\]

    Applying the cyclic property of the trace operator, we get:
    \[= \arg\max_{P\in\mathcal{P}}  \tr(XX^T P YY^TP^T)\]

    Defining the kernels \(K_X = XX^T, K_Y = YY^T\), that is
    \[= \arg\max_{P\in\mathcal{P}}  \tr(K_X P K_Y P^T)\]

    The expression above is the exact standard Koopmans-Beckmann formulation of the QAP with matrices \(K_X\) and \(K_Y\). Hence:
    \[\text{MCP}_{S_F} = QAP(XX^T, YY^T)\]

\end{proof}

\subsection{Equivalence of \(\text{MCP}_{O(d)}\) and \(\text{MCP}_{S_F}\)}

Recall that \(M_P = X^T P Y\) denotes the cross-correlation matrix of \(X\) and \(Y\) under permutation \(P\). Prop. \ref{prop:od_nuclear} translates into an optimization over \(\sigma(M_P)\), the singular values of \(M_P\):

\begin{equation} \label{eq:od_svals}
    \mcp_{O(d)}(X, Y) = \arg\max_{P\in\mathcal{P}} \sum_{i=1}^d \sigma_i(M_P),
\end{equation}

Denote \(P_O = \mcp_{O(d)}\) and \(P \ne \mcp_{O(d)}\), for brevity. By Def. \ref{def:mcpf} and Eq. \eqref{eq:od_svals}, the sum of singular values under \(P_O\) is strictly maximal:
\begin{equation}
    \sum_{i=1}^d \sigma_i(M_{P_O}) > \sum_{i=1}^d\sigma_i(M_{P})
\end{equation}

In Assumption \ref{assum:weak_majorization} below, we assume this inequality holds not just for the total sum, but is preserved across all partial sums.

\newcounter{oldAssumption}
\setcounter{oldAssumption}{\value{assumption}}
\setcounter{assumption}{0}

\begin{assumption} \label{assum:weak_majorization}
    Let \(P_O \in \mcp_{O(d)}\) and \(P \notin \mcp_{O(d)}\), then \(\sigma(M_{P_O})\) weakly majorizes \(\sigma(M_{P})\). Namely,
    \[\sum_{i=1}^m \sigma_i(M_{P_O}) \geq \sum_{i=1}^m \sigma_i(M_{P})\quad \forall m \in \{1, 2, \dots, d\}\]
\end{assumption}

\setcounter{assumption}{\value{oldAssumption}}

Assumption \ref{assum:weak_majorization} imposes a relatively mild constraint on the spectral decay of the permuted cross-correlation matrices. In the following propositions, we demonstrate its theoretical validity across two indicative scenarios. First, Prop. \ref{prop:finite_assum_support} establishes that in the finite-sample regime, the assumption holds strictly if \(Y\) is an isometry of \(X\) (\(Y = XQ\) for \(Q \in O(d)\)). Second, transitioning to the asymptotic regime, Prop. \ref{prop:infinite_assum_support} proves that as the sample size \(n\to\infty\), the \(\mcp_{O(d)}\) pairing weakly majorizes any uniformly sampled permutation with probability 1.

\begin{proposition} \label{prop:finite_assum_support}
    Let \(P'Y = XQ\) for \(P' \in \mathcal{P}, Q \in O(d)\). Then Assum. \ref{assum:weak_majorization} holds.
\end{proposition}

\begin{proof}
    Let \(M_P = X^T P Y = X^T P P'^{-1} XQ\). Because \(X^T X = I_d\), and \(P, P', Q\) are orthogonal matrices, their induced \(2\)-norms are all \(1\). By sub-multiplicativity, the spectral norm of \(M_P\) is strictly bounded:
    \[\sigma_1(M_P) = \lVert X^T P P'^{-1}XQ \rVert_2 \leq \lVert X \rVert_2^2 \lVert P \rVert_2 \lVert P'^{-1} \rVert_2 \lVert Q \rVert_2 = 1\]
    This indicates that \(\sigma_i(M_P) \leq \sigma_1(M_P) \leq 1\) for all \(i \in \{1, \dots, d\}\) and any permutation \(P\).

    Now, notice that for the permutation \(P' \in \mathcal{P}\) we get \(M_{P'} = X^T P' P'^{-1} X Q = X^T X Q = Q\). Since \(Q \in O(d)\), its singular values are exactly \(\sigma_i(M_{P'}) = 1\) for all \(i\).

    By Prop. \ref{prop:od_nuclear}, \(P_O\) maximizes the nuclear norm \(\sum_{i=1}^d \sigma_i(M_P)\). Since every \(\sigma_{i} \le 1\), the theoretical upper bound of this sum is exactly \(d\). Because the permutation \(P'\) successfully achieves this strict upper bound, it follows that the global maximizer \(P_{O}\) must also achieve it, strictly requiring that \(\sigma_{i}(M_{P_{O}}) = 1\) for all \(i\).

    Therefore, for any \(m \in \{1, \dots, d\}\) and any \(P\), it follows that:
    \[\sum_{i=1}^m \sigma_i(M_{P_O}) = m \geq \sum_{i=1}^m \sigma_i(M_P)\]
\end{proof}

\begin{lemma} \label{lemma:assum_by_n}
    Let \(P \sim \mathcal{U}(\mathcal{P})\) be a permutation matrix chosen uniformly at random from the permutation group. Then, for any \(m \in \{1, \dots, d\}\):
    \[\mathbb{P}\left(\sum_{i=1}^m\sigma_i(M_{P_O}) > \sum_{i=1}^m\sigma_i(M_{P}) \right) = 1 - \frac{1}{\log(n)}\]
\end{lemma}

\begin{proof}
    First note that because \(X\) and \(Y\) are non-degenerate whitened matrices, the optimal pairing \(P_O\) that maximizes the sum of singular values of the permuted cross-correlation matrix \(M_P = \frac{1}{n}X^T P Y\) is guaranteed to capture cross-correlation. Consequently, the leading singular value of the optimal cross-correlation matrix, \(\sigma_1(M_{P_O})\), is naturally bounded away from zero by a constant \(\gamma > 0\) that is independent of \(n\). As singular values are non-negative, for all \(m\), the partial sum of the first \(m\) singular values of \(M_{P_O} = \frac{1}{n}X^T P Y\) is bounded away from zero by a constant independent of \(n\):
    \[\sum_{i = 1}^m \sigma_i(M_{P_O}) = \gamma_m > 0\]
    
    We begin by evaluating the expected squared Frobenius norm of the cross-correlation matrix under a random permutation, \(\mathbb{E}_P[\lVert M_P \rVert_F^2]\).

    By definition, the element at the \(k\)-th row and \(l\)-th column of the unnormalized matrix \(nM_P\) is given by:
    \[\left( X^TPY \right)_{kl} = \sum_{i=1}^n X_{ik}Y_{\pi(i)l}\]
    where \(\pi\) is the permutation function corresponding to \(P\). We compute the expected square of this entry:
    \[\mathbb{E}_P\left[ \left( X^TPY \right)^2_{kl} \right] = \mathbb{E}_\pi\left[ \sum_{i=1}^n\sum_{j=1}^n X_{ik}Y_{\pi(i)l}X_{jk}Y_{\pi(j)l} \right]\]

    We split this double sum into diagonal (\(i = j\)) and off-diagonal (\(i \neq j\)) terms:
    \[\mathbb{E}_P\left[ \left( X^TPY \right)^2_{kl} \right] = \sum_{i=1}^n X_{ik}^2 \mathbb{E}_\pi\left[ Y^2_{\pi(i)l} \right] + \sum_{i\neq j} X_{ik}X_{jk} \mathbb{E}_\pi\left[ Y_{\pi(i)l} Y_{\pi(j)l} \right]\]

    Because \(\pi(i)\) is uniformly distributed over \(\{1, \dots, n\}\), the expectation of the squared term is simply the population average, which by whitening is:
    \[\mathbb{E}_\pi\left[ Y^2_{\pi(i)l} \right] = \frac{1}{n}\sum_{t=1}^n Y_{tl}^2 = 1\]

    For the off-diagonal term, the indices \(\pi(i)\) and \(\pi(j)\) are drawn uniformly without replacement. Therefore:
    \[\mathbb{E}_\pi\left[ Y_{\pi(i)l} Y_{\pi(j)l} \right] = \frac{1}{n(n-1)}\sum_{s\neq t}Y_{sl}Y_{tl}\]

    We can rewrite the sum of the cross-product using the square of the total sum
    \[\sum_{s\neq t}Y_{sl}Y_{tl} = \left( \sum_{s=1}^n Y_{sl} \right)^2 - \sum_{s=1}^n Y^2_{sl}\]

    By the whitening, \(\sum_{s=1}^n Y_{sl} = 0\). Therefore, the sum of the cross-products simplifies to \(-n\). This yields:
    \[\mathbb{E}_\pi\left[ Y_{\pi(i)l} Y_{\pi(j)l} \right] = \frac{-n}{n(n-1)} = \frac{-1}{n-1}\]

    Similarly, because \(X\) is whitened, we know that \(\sum_{i} X_{ik}^2 = n\) and \(\sum_{i\neq j} X_{ik}X_{jl} = -n\). Substituting these back into the expectation:
    \[\mathbb{E}_P\left[ \left( X^TPY \right)^2_{kl} \right] = n\cdot 1 + (-n)\left( \frac{-1}{n-1} \right) = n + \frac{n}{n-1} = \frac{n^2}{n-1}\]

    Normalizing by \(\frac{1}{n^2}\) to recover \(M_P\), the expected squared entry is:
    \[\mathbb{E}_P\left[ (M_P)^2_{kl} \right] = \frac{1}{n-1}\]

    Summing over all \(d\times d\) elements, the expected squared Frobenius norm is strictly:
    \[\mathbb{E}_P\left[ \lVert M_P\rVert^2_F \right] = \frac{d^2}{n-1}\]

    Now, by Markov's inequality, we can bound the probability that the Frobenius energy of the random pairing is large. For any \(t > 0\):
    \[\mathbb{P}\left( \lVert M_P \rVert^2_F \geq t\frac{d^2}{n-1} \right) \leq \frac{1}{t}\]

    Let us choose a slowly growing function, such as \(t = \log(n)\), to ensure the bound becomes tight while the probability of failure approaches zero. Thus, with high probability, \(\lVert M_P \rVert_F = O\left(\sqrt{\frac{\log(n)}{n}} \right)\).

    We now relate the Frobenius norm to the partial sums of singular values using the Cauchy-Schwarz inequality. For any \(m \leq d\):
    \[\sum_{i=1}^m \sigma_i(M_P) \leq \sqrt{m} \sqrt{\sum_{i=1}^m \sigma_i^2(M_P)} \leq \sqrt{m}\lVert M_P \rVert_F\]

    Consequently, for a random permutation, the partial sum of singular values decays as:
    \[\sum_{i=1}^m \sigma_i(M_P) \leq d\sqrt{\frac{m\log(n)}{n-1}}\]

    We assumed \(\sum_{i=1}^n \sigma_i(M_{P_O}) = \gamma_m > 0\).
    
    As the bound from the equation above strictly decreases towards \(0\) as \(n\to \infty\), there exist sufficiently large \(N\) such that for all \(n > N\), \(\gamma_m > d\sqrt{\frac{m\log(n)}{n-1}}\).

    The probability from above was bounded by \(\frac{1}{t} = \frac{1}{\log(n)}\), hence, for any \(m \leq d\),
    \[\mathbb{P}\left(\sum_{i=1}^m\sigma_i(M_{P_O}) > \sum_{i=1}^m\sigma_i(M_{P}) \right) = 1 - \frac{1}{\log(n)}\]
\end{proof}

\begin{proposition} \label{prop:infinite_assum_support}
    Let \(P \sim \mathcal{U}(\mathcal{P})\) be a permutation matrix chosen uniformly at random from the permutation group. Then:
    \[\lim_{n\to\infty}\mathbb{P}\left( \sigma(M_{P_O}) \ {}_{\mathrm{w}}\!\succ \sigma(M_P) \right) = 1\]
\end{proposition}

\begin{proof}
    Let \(E_m\) denote the event that the strict partial sum inequality holds for a specific \(m\):
    \[E_m = \left\{ P \in \mathcal{P}_n | \gamma_m > \sum_{i=1}^m \sigma_i(M_P) \right\}\]
    where \(\gamma_m = \sum_{i=1}^m \sigma_i(M_{P_O})\) is the strictly positive signal from the optimal pairing.

    The event that the optimal pairing \(P_O\) weakly majorizes a random permutation \(P\) corresponds to the intersection of these events across all \(d\) dimensions:
    \[\mathbb{P}\left( \sigma(M_{P_O}) \succ_{\mathrm{w}} \sigma(M_P) \right) = \mathbb{P} \left( \bigcap_{m=1}^d E_m \right)\]

    By Lemma \ref{lemma:assum_by_n}, for any \(m\), from sufficiently large \(n\):
    \[\mathbb{P}(E_m^c) \leq \frac{1}{\log(n)}\]

    To find the probability that all \(E_m\) hold simultaneously, we bound the probability of the complement using the union bound:
    \[\mathbb{P}\left( \bigcup_{m=1}^d E_m^c \right) \leq \sum_{m=1}^d \mathbb{P}(E_m^c)\]

    Substitute our individual failure probability into the sum:
    \[\mathbb{P}\left( \bigcup_{m=1}^d E_m^c \right) \leq \sum_{m=1}^d \frac{1}{\log(n)} = \frac{d}{\log(n)}\]

    The probability of simultaneous weak majorization is the complement of the union of failures:
    \[\mathbb{P}\left( \bigcap_{m=1}^d E_m \right) = 1 - \mathbb{P}\left( \bigcup_{m=1}^d E_m^c \right) = 1 - \frac{d}{\log(n)}\]

    Crucially, the feature dimension \(d\) is constant with respect to the sample size \(n\). Therefore, as \(n\to\infty\), the term \(\frac{d}{\log(n)}\) strictly vanishes. Meaning that,
    \[\lim_{n\to\infty}\mathbb{P}\left( \sigma(M_{P_O}) \ {}_{\mathrm{w}}\!\succ \sigma(M_P) \right) = 1\]
\end{proof}

To synthesize these findings and prove Thm. \ref{thm:mcp_equiv}. we first recall the formal definition and two critical properties of weak-majorization, as outlined in \cite{marshall1979inequalities}.

\begin{definition}
    (Def. 1.A.2 in \cite{marshall1979inequalities}) Let \(x, y \in \mathbb{R}^d\) be two vectors. Then, \(x\) \underline{weakly majorizes} \(y\) if
    \[\sum_{i=1}^m x_i \geq \sum_{i=1}^m y_i\quad \forall m\in \{1, \dots, d\}\]
    Denoted \(y \prec_{\mathrm{w}} x\).
\end{definition}

\begin{proposition} \label{prop:weak_major_ineq}
    (Prop. 4.B.2 in \cite{marshall1979inequalities}) \(y \prec_{\mathrm{w}} x\) if and only if 
    \[\sum_{i=1}^d g(x_i) \geq \sum_{i=1}^d g(y_i)\]
    for all continuous increasing convex functions \(g: \mathbb{R} \to \mathbb{R}\).
\end{proposition}

\begin{proposition} \label{prop:weak_major_strict}
    (Thm. 3.A.8 in \cite{marshall1979inequalities}) If \(y \prec_{\mathrm{w}} x\), and \(\sum_i x_i > \sum_i y_i\), then 
    \[g(x) > g(y)\]
    for any \(g:\mathcal{D} \to \mathbb{R}\) strictly increasing and strictly Schur-convex, \(\mathcal{D} \subseteq \mathbb{R}^d\).
\end{proposition}

Given these properties, we can prove Thm. \ref{thm:mcp_equiv}.

\setcounter{oldTheorem}{\value{theorem}}
\setcounter{theorem}{1}

\begin{theorem} \label{thm:mcp_equiv}
    Under Assumption \ref{assum:weak_majorization},
    \[\text{MCP}_{O(d)} = \text{MCP}_{S_F}\]
\end{theorem}

\setcounter{theorem}{\value{oldTheorem}}

\begin{proof}
   We start by proving that \(P_O\) strictly maximizes the Frobenius objective. Assum. \ref{assum:weak_majorization} establishes that the vector \(\sigma(P_O)\) weakly majorizes the vector \(\sigma(P)\).

    Let \(g(x) = x^2\). For the domain \(x \geq 0\) (which applies to all singular values), \(g(x)\) is a continuous increasing convex function.

    By Prop. \ref{prop:weak_major_ineq}, applying this to our weakly majorized singular value spectra yields:
    \[\sum_{i=1}^d \sigma_i^2(P_O) \geq \sum_{i=1}^d \sigma_i^2(P)\]

    Hence:
    \[\lVert M_{P_O} \rVert_F^2 = \sum_{i=1}^d \sigma_i^2(P_O) \geq \sum_{i=1}^d \sigma_i^2(P) = \lVert M_{P} \rVert_F^2\quad \forall P \neq P_O\]

    Thus, \(P_O\) is the global maximum for the Frobenius norm objective, meaning that
    \[P_O \in \text{MCP}_{S_F}\]

    The other inclusion is similar. Take a permutation \(P \notin \text{MCP}_{O(d)}\). Then, its Nuclear norm is strictly less than that of \(P_O \in \text{MCP}_{O(d)}\). That is,
    \[\sum_{i=1}^d \sigma_i(P_O) > \sum_{i=1}^d \sigma_i(P)\]

    Notice that \(g(x) = \sum_{i}x_i^2\) in the domain \(x \in \mathbb{R}_+^n\) is strictly increasing and strictly Schur-convex. By Prop. \ref{prop:weak_major_strict},
    \[\sum_{i=1}^d \sigma_i^2(P_O) > \sum_{i=1}^d \sigma_i^2(P)\]

    This means that \(P \notin \text{MCP}_{S_F}\). Finally,
    \[\text{MCP}_{O(d)} = \text{MCP}_{S_F}\]
\end{proof}

\section{Additional Results} \label{app:results}

\subsection{Fig. \ref{fig:tc} Results}

Tab. \ref{tab:app:tc} presents the full results visualized in Fig. \ref{fig:tc}.

\begin{table}[h]
\centering
\caption{\textbf{Total Correlation Results} (mean $\pm$ std). Bold values indicate the best performance. The paired CCA results are provided in gray at the bottom as an upper-bound reference.}

\definecolor{lightgreen}{HTML}{E2EFDA}

\begin{adjustbox}{width=\textwidth}
\begin{tabular}{lcccccc}
\toprule
Method & PC-KL & PC-PA & KL-PA & SNARE & Flickr & COCO \\
\midrule
J-MDS-CCA & \( -0.110 \pm 0.234 \) & \( 0.048 \pm 0.180 \) & \( 0.751 \pm 0.258 \) & \( -0.014 \pm 0.621 \) & \( 0.020 \pm 0.062 \) & \( -0.041 \pm 0.100 \) \\
SCOTv1-CCA & \( -0.136 \pm 0.000 \) & \( 0.042 \pm 0.000 \) & \( 0.259 \pm 0.000 \) & \( 0.970 \pm 0.000 \) & \( 0.026 \pm 0.068 \) & \( 0.132 \pm 0.000 \) \\
SCOTv2-CCA & \( 0.300 \pm 0.000 \) & \( 0.107 \pm 0.000 \) & \( -0.148 \pm 0.000 \) & \( 0.816 \pm 0.000 \) & \( 0.030 \pm 0.040 \) & \( -0.005 \pm 0.000 \) \\
SCA & \( -0.093 \pm 0.126 \) & \( 0.024 \pm 0.123 \) & \( 0.045 \pm 0.113 \) & \( 0.144 \pm 0.814 \) & \( -0.018 \pm 0.078 \) & \( -0.010 \pm 0.096 \) \\
UCA & \( -0.101 \pm 0.383 \) & \( -0.282 \pm 0.296 \) & \( 0.493 \pm 0.741 \) & \( 0.883 \pm 0.795 \) & \( -0.081 \pm 0.468 \) & \( 0.173 \pm 0.460 \) \\
\rowcolor[HTML]{E2EFDA}
UCCA (ours) & \( \mathbf{1.239} \pm 0.442 \) & \( \mathbf{1.417} \pm 0.029 \) & \( \mathbf{1.957} \pm 0.012 \) & \( \mathbf{1.755} \pm 0.031 \) & \( \mathbf{0.751} \pm 0.086 \) & \( \mathbf{1.371} \pm 0.063 \) \\
\midrule
\textcolor{gray}{Paired CCA} & \textcolor{gray}{\( 1.979 \)} & \textcolor{gray}{\( 1.993 \)} & \textcolor{gray}{\( 1.990 \)} & \textcolor{gray}{\( 1.850 \)} & \textcolor{gray}{\( 1.574 \)} & \textcolor{gray}{\( 1.768 \)} \\
\bottomrule
\end{tabular}
\end{adjustbox}
\label{tab:app:tc}
\end{table}

\subsection{Full Cross-Classification Results} \label{app:classification}

Tab. \ref{tab:app:cross_class_knn21} complements the primary cross-view classification findings in Tab. \ref{tab:classification} by presenting the alternate direction of the kNN cross-view classification. The results are consistent with those presented in Tab. \ref{tab:classification}. \ourmethod\ consistently achieves the best performance, with the exception of the COCO dataset, where it yields the second-best result. Also, note that replacing the kNN classifier with a linear kernel preserves all observed performance trends.

\begin{table}[h]
\centering
\caption{\textbf{Cross-view classification results.} Mean accuracy ($\pm$ std) of a kNN classifier fitted on the learned embeddings of the second view and evaluated on the first. This table complements the results in the main text by showing the alternate classification direction. Bold and underlined values indicate the best and second-best performances, respectively. Paired CCA results are provided in gray at the bottom.}

\begin{adjustbox}{width=\textwidth}
\begin{tabular}{lccccc}
\toprule
Method & PC-KL & PC-PA & KL-PA & SNARE & COCO \\
\midrule
J-MDS-CCA & \( 0.090 \pm 0.030 \) & \( 0.100 \pm 0.012 \) & \( 0.095 \pm 0.011 \) & \( 0.198 \pm 0.076 \) & \( 0.123 \pm 0.113 \) \\
SCOTv1-CCA & \( 0.066 \pm 0.000 \) & \( 0.082 \pm 0.000 \) & \( 0.113 \pm 0.000 \) & \( \underline{0.809} \pm 0.000 \) & \( 0.156 \pm 0.000 \) \\
SCOTv2-CCA & \( \underline{0.105} \pm 0.000 \) & \( 0.102 \pm 0.000 \) & \( 0.133 \pm 0.000 \) & \( 0.805 \pm 0.000 \) & \( \mathbf{0.301} \pm 0.000 \) \\
SCA & \( 0.093 \pm 0.020 \) & \( \underline{0.115} \pm 0.028 \) & \( 0.104 \pm 0.020 \) & \( 0.335 \pm 0.298 \) & \( 0.148 \pm 0.021 \) \\
UCA & \( \underline{0.105} \pm 0.065 \) & \( 0.064 \pm 0.037 \) & \( \underline{0.163} \pm 0.112 \) & \( 0.364 \pm 0.273 \) & \( 0.139 \pm 0.046 \) \\
\rowcolor{lightgreen}
UCCA (ours) & \( \mathbf{0.290} \pm 0.067 \) & \( \mathbf{0.332} \pm 0.014 \) & \( \mathbf{0.540} \pm 0.022 \) & \( \mathbf{0.818} \pm 0.020 \) & \( \underline{0.208} \pm 0.016 \) \\
\midrule
\textcolor{gray}{Paired CCA} & \textcolor{gray}{\( 0.551 \pm 0.000 \)} & \textcolor{gray}{\( 0.543 \pm 0.000 \)} & \textcolor{gray}{\( 0.602 \pm 0.000 \)} & \textcolor{gray}{\( 0.899 \pm 0.000 \)} & \textcolor{gray}{\( 0.232 \pm 0.007 \)} \\
\bottomrule
\end{tabular}
\end{adjustbox}
\label{tab:app:cross_class_knn21}
\end{table}

\subsection{Validating \(\bf\textbf{MCP}_{O(d)} = \textbf{MCP}_{S_F}\)} \label{app:valid_thm2_exp}

\begin{figure}[h]
    \centering
    \includegraphics[width=\linewidth]{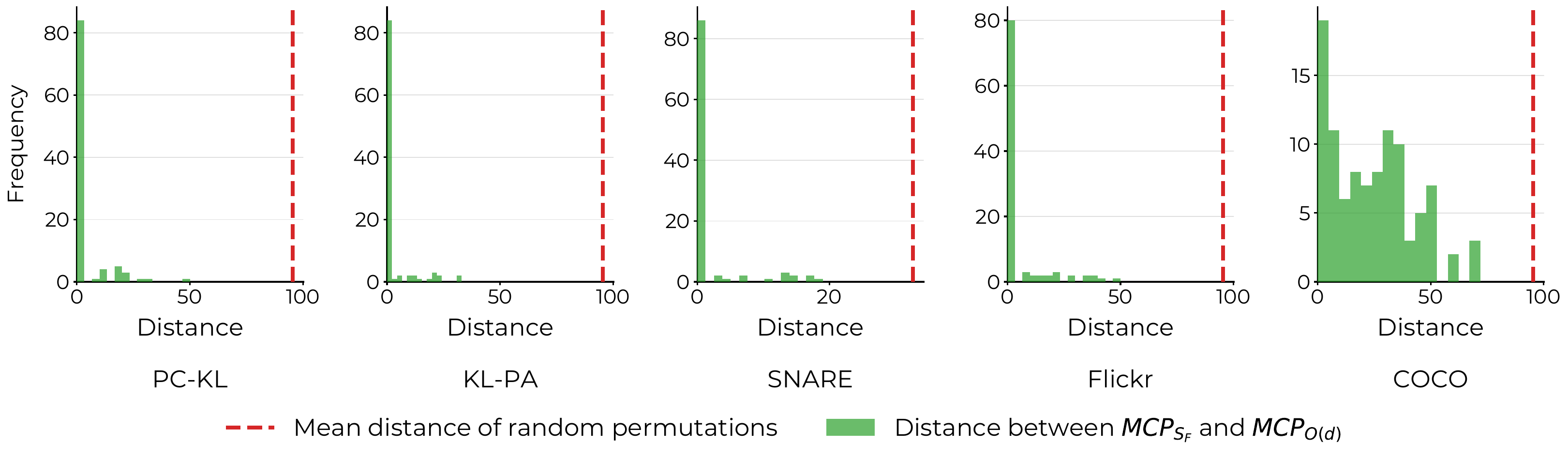}
    \caption{\textbf{Extended validation of \(\bf\textbf{MCP}_{O(d)} = \textbf{MCP}_{S_F}\).} Distance between \(\mcp_{S_F}\) and the approximated \(\text{MCP}_{O(d)}\) (denoted as $\tilde{P}_O$), computed on the anchors of the each dataset across 100 initializations. This figure complements Fig. \ref{fig:mcp_equiv} in the main text by demonstrating that the strict concentration at zero holds consistently across all evaluated domains, further validating Thm. \ref{thm:mcp_equiv}.}
    \label{fig:app:mcp_equiv}
\end{figure}

This section provides further details on the empirical validation of Thm. \ref{thm:mcp_equiv} in Sec. \ref{sec:results} and Fig. \ref{fig:mcp_equiv}, which asserts that \(\text{MCP}_{O(d)} = \text{MCP}_{S_F}\) holds in practice. As introduced in the main text, we compare \(\mcp_{S_F}\), obtained via an approximate QAP solver, against \(\tilde{P}_O\), which acts as an estimated projection of \(\mcp_{S_F}\) onto the \(\text{MCP}_{O(d)}\) solution space.

This projection is computed by sequentially solving the dual maximizations formulated in Eq. \eqref{eq:mcp_f}. First, given \(\mcp_{S_F}\), we isolate the optimal orthogonal transformation:
\begin{equation*}
    Q_S = \arg\max_{Q\in O(d)} \tc(X, \mcp_{S_F}Y)
\end{equation*}

Subsequently, we compute the maximum correlation permutation (MCP) under this fixed transformation:
\begin{equation*}
    \tilde{P_O} = \mcp(XQ_S, Y)
\end{equation*}

Crucially, because the orthogonal group is a subset of the scaled Frobenius sphere (\(O(d) \subseteq \sqrt{d}S_F\)), if \(\tilde{P_O} = \mcp_{S_F}\), the approximation is exact. This yields \(\tilde{P}_O = \text{MCP}_{O(d)} = \text{MCP}_{S_F}\).

Fig. \ref{fig:app:mcp_equiv} complements the PC-PA results from Fig. \ref{fig:mcp_equiv} by detailing the Kendall Tau distances for the remaining datasets. Across all domains, the distances remain heavily concentrated at zero, providing robust, dataset-agnostic empirical support for Thm. \ref{thm:mcp_equiv}.

\subsection{Running Times} \label{app:runtime}

\begin{figure}
    \centering
    \includegraphics[width=\linewidth]{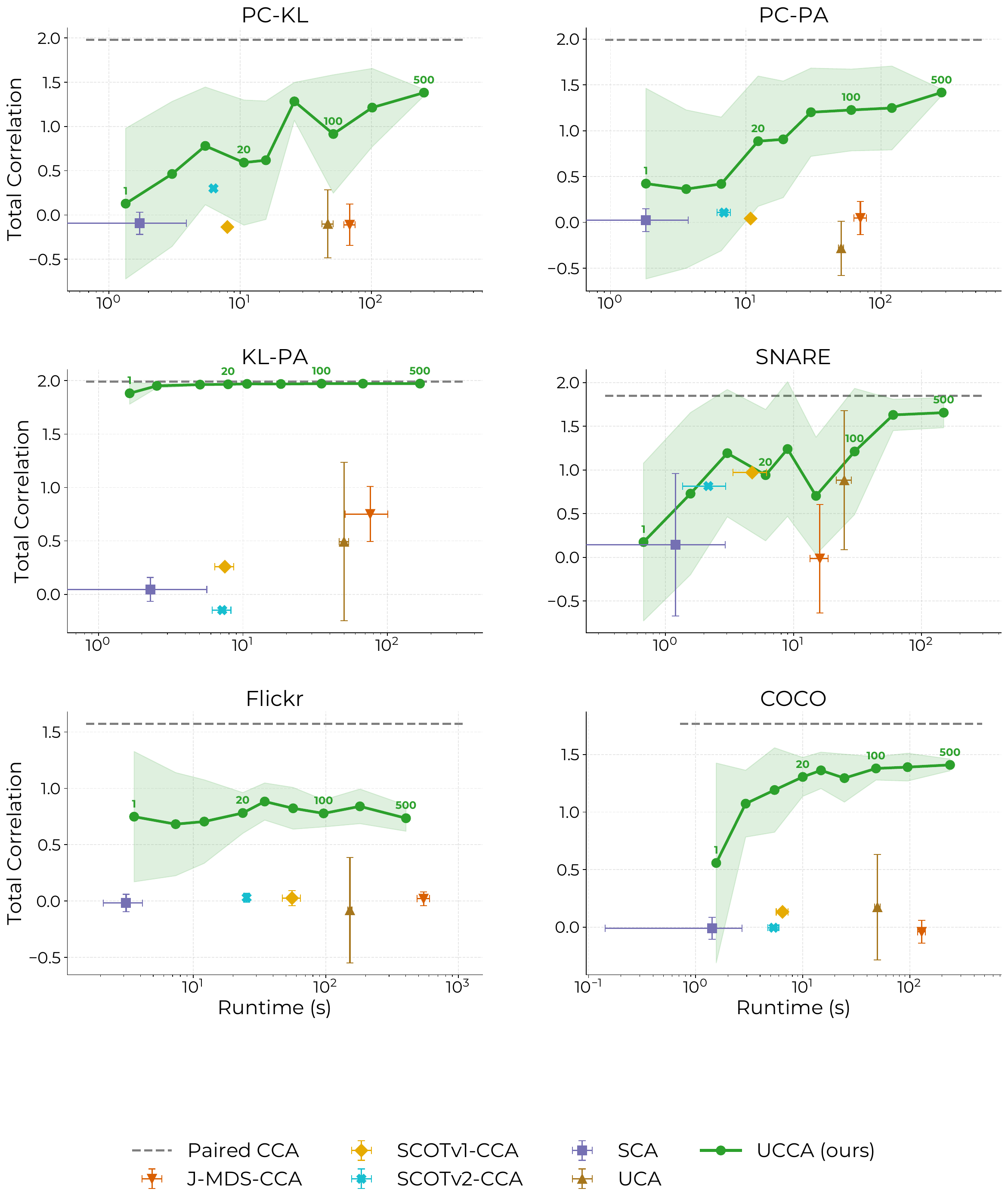}
    \caption{\textbf{Extended runtime and performance analysis.} Total correlation versus runtime (in seconds, log scale) for \ourmethod\ and baseline methods across the various datasets. The green numerical annotations on the \ourmethod\ curve denote the number of clustering repetitions. \ourmethod\ rapidly surpasses baseline performance even at a lower repetition setting and improves as repetitions increase, while maintaining a practical total execution time.}
    \label{fig:app:runtime}
\end{figure}

This section provides a more detailed analysis of the computational efficiency of \ourmethod, expanding upon the results presented in Fig. \ref{fig:runtime}. A primary hyperparameter dictating the computational overhead of our method is the number of clustering repetitions used to extract these anchor points. To systematically evaluate the trade-off between execution time and alignment accuracy, we tracked the TC and total runtime (in seconds, plotted on a logarithmic scale) across varying numbers of clustering repetitions. These intervals (e.g., 1, 20, 100, and 500) are explicitly annotated on the \ourmethod\ performance curves.

Fig. \ref{fig:app:runtime} illustrates these trade-offs for all six datasets, complementing the findings from Fig. \ref{fig:runtime}. Across all evaluated domains, \ourmethod\ exhibits a highly favorable scalability profile. Even at the lowest repetition settings (e.g., 1 or 20), the method reliably outperforms the competing unsupervised baselines, often executing in under a minute. As the number of clustering repetitions increases, the variance across runs decreases, and the alignment quality steadily converges toward the paired CCA upper bound (represented by the gray dashed horizontal line). Crucially, even when scaled to 500 repetitions to maximize stability and performance, the total execution time remains highly practical, rarely exceeding a few minutes.

\subsection{Direct Optimization of $\mcp_{O(d)}$} \label{app:direct_optimization}

To validate the hardness of directly optimizing the $\mcp_{O(d)}$, and justifying the $\mcp_{S_F}$ relaxation, we implemented an EM-style alternating maximization approach for directly optimizing the \(\text{MCP}_{O(d)}\) objective (Eq. \eqref{eq:mcp_f}). We alternate between (1) CCA optimization over A,B with fixed P, and (2) Linear Sum Assignment (Hungarian algorithm) optimization over P with fixed A,B.

Additionally, we compared against two variants of ACCA \cite{cheng2023towards}, which similarly alternate between CCA and assignment optimization.

The results are summarized in Tab. \ref{tab:app:direct_optimization}. All alternating optimization methods were repeated 50 times with different initializations, as done for the 2-opt solver in UCCA. These results indicate that the QAP formulation provides a more robust optimization landscape than directly optimizing the original $\mcp_{O(d)}$ objective.

\begin{table}[h]
\centering
\caption{\textbf{Total Correlation comparison to direct optimization of $\mcp_{O(d)}$.} Total correlation results (mean \(\pm\) std). Bold values indicate the best unsupervised performance, and the paired CCA results are provided in gray as an upper-bound reference. The QAP formulation provides a more robust optimization landscape than directly optimizing the original $\mcp_{O(d)}$ objective.}

\begin{adjustbox}{width=0.6\textwidth}
\begin{tabular}{lcc}
\toprule
Method & PC-PA & COCO \\
\midrule
ACCA & $0.035 \pm 0.171$ & $0.057 \pm 0.109$ \\
ACCA-anchors & $-0.066 \pm 0.459$ & $0.339 \pm 0.321$ \\
UCCA (EM) & $0.826 \pm 0.602$ & $0.420 \pm 0.357$ \\
\rowcolor{lightgreen}
UCCA (QAP 2-opt, ours) & $\mathbf{1.417} \pm 0.029$ & $\mathbf{1.371} \pm 0.063$ \\
\midrule
\textcolor{gray}{Paired CCA} & \textcolor{gray}{$1.993$} & \textcolor{gray}{$1.768$} \\
\bottomrule
\end{tabular}
\end{adjustbox}
\label{tab:app:direct_optimization}
\end{table}

\subsection{Effect of QAP Solver}

\begin{figure}[h]
    \centering
    \includegraphics[width=0.5\linewidth]{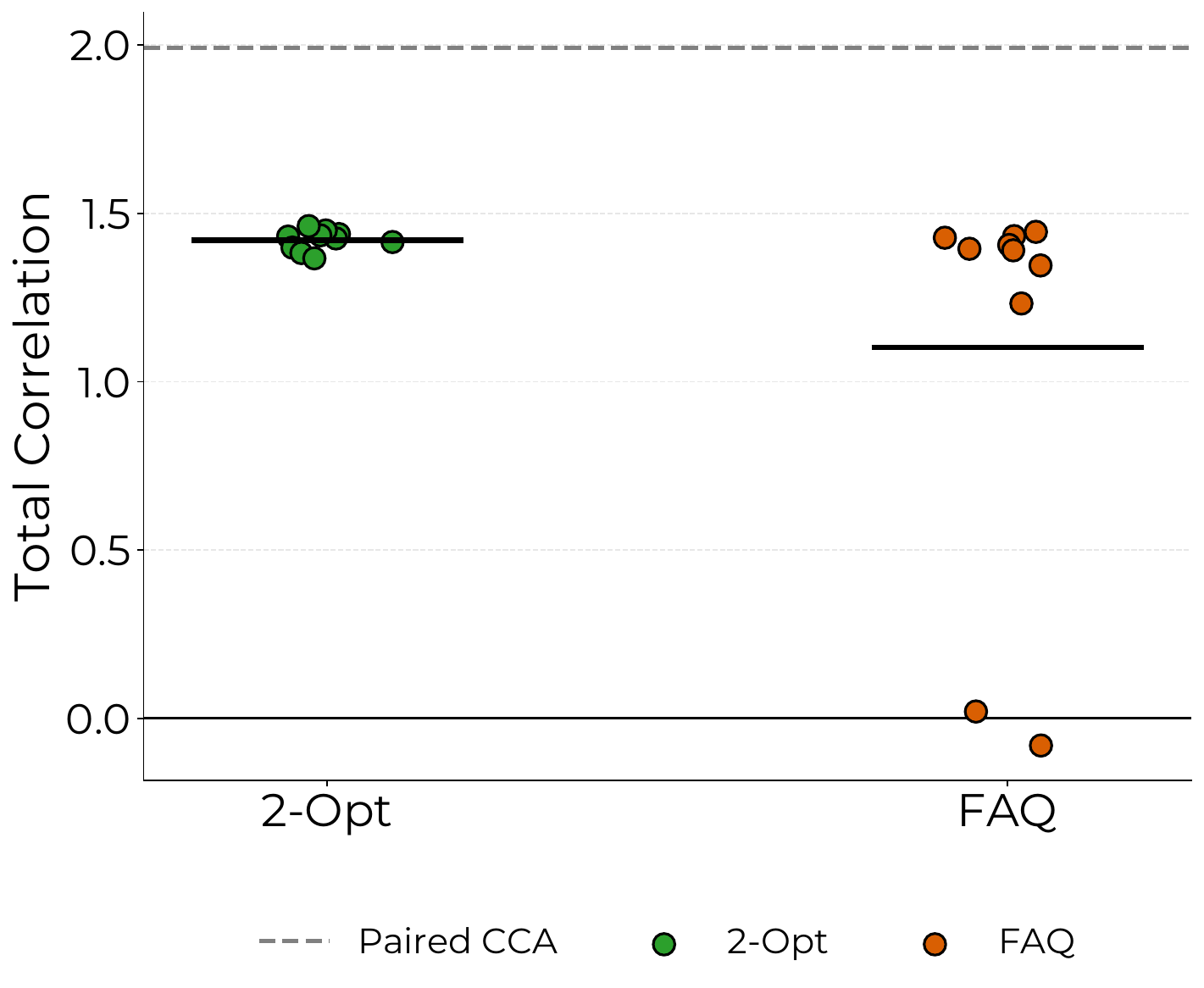}
    \caption{\textbf{Effect of the QAP solver on performance.} Total Correlation achieved across multiple initializations on the PC-PA Handwritten benchmark, using the 2-Opt and FAQ approximate QAP solvers. Both solvers yield very similar overall performance and alignment quality. Because FAQ failed to converge in two instances (points near zero), 2-Opt was selected as the default solver for \ourmethod\ simply to ensure consistent execution across all runs.}
    \label{fig:app:qap_solver}
\end{figure}

The structural matching step of \ourmethod\ relies on solving the QAP. Because the exact QAP is NP-hard, we employ efficient approximate solvers to ensure practical scalability. In this section, we evaluate the impact of the chosen solver on the final alignment performance. Specifically, we compare two popular options: 2-Opt \cite{heider1972computationally} and the Fast Approximate QAP (FAQ) algorithm \cite{vogelstein2015fast}.

Fig. \ref{fig:app:qap_solver} visualizes the Total Correlation achieved by both solvers across multiple initializations. Overall, the choice of solver does not significantly alter the performance; when FAQ converges successfully, it yields an alignment quality almost identical to that of 2-Opt. However, FAQ failed to converge in two distinct instances (indicated by the points clustering near zero). Because of these rare convergence issues, we selected 2-Opt as the default solver for \ourmethod\ simply to ensure consistent and reproducible alignments across all runs.

\subsection{Comparison with Non-linear UCA} \label{app:uca_nonlinear}

As noted in the main text, \ourmethod\ is fundamentally designed to extract linear projections that maximize the correlation between unaligned views. To ensure a fair architectural comparison, the primary evaluation benchmarks \ourmethod\ against the linear formulation of UCA \cite{hoshen2018unsupervised}. However, UCA is also capable of utilizing deep neural networks to learn non-linear mappings. This section extends our evaluation to compare \ourmethod\ against this non-linear variant of UCA.

Tab. \ref{tab:app:uca_nonlinear_tc} and \ref{tab:app:uca_nonlinear_classification} present the Total Correlation and cross-view classification results, respectively. Despite the increased expressive capacity afforded by non-linear transformations, the non-linear UCA baseline struggles to discover meaningful structural alignments across most datasets. \ourmethod\ consistently outperforms it by a significant margin across all metrics. These results demonstrate that the explicit geometric alignment strategy and QAP-based skeleton matching employed by \ourmethod\ are vastly more robust and effective for the Unpaired CCA scenario.

\begin{table}[h]
\centering
\caption{\textbf{Total Correlation comparison with non-linear UCA.} Total correlation results (mean \(\pm\) std). Despite its non-linear capacity, \ourmethod\ significantly outperforms non-linear UCA across all datasets. Bold values indicate the best unsupervised performance, and the paired CCA results are provided in gray as an upper-bound reference.}

\begin{adjustbox}{width=\textwidth}
\begin{tabular}{lcccccc}
\toprule
Method & PC-KL & PC-PA & KL-PA & SNARE & Flickr & COCO \\
\midrule
UCA (nonlinear) & \( 0.378 \pm 0.552 \) & \( 0.262 \pm 0.435 \) & \( -0.414 \pm 0.532 \) & \( 1.066 \pm 0.962 \) & \( 0.072 \pm 0.532 \) & \( 0.321 \pm 0.393 \) \\
\rowcolor{lightgreen}
UCCA (ours) & \( \mathbf{1.239} \pm 0.442 \) & \( \mathbf{1.417} \pm 0.029 \) & \( \mathbf{1.957} \pm 0.012 \) & \( \mathbf{1.755} \pm 0.031 \) & \( \mathbf{0.751} \pm 0.086 \) & \( \mathbf{1.371} \pm 0.063 \) \\
\midrule
\textcolor{gray}{Paired CCA} & \textcolor{gray}{\( 1.979 \)} & \textcolor{gray}{\( 1.993 \)} & \textcolor{gray}{\( 1.990 \)} & \textcolor{gray}{\( 1.850 \)} & \textcolor{gray}{\( 1.574 \)} & \textcolor{gray}{\( 1.768 \)} \\
\bottomrule
\end{tabular}
\end{adjustbox}
\label{tab:app:uca_nonlinear_tc}
\end{table}

\begin{table}[h]
\centering
\caption{\textbf{Cross-view classification comparison with non-linear UCA.} Mean accuracy ($\pm$ std) of a kNN classifier trained on the learned embeddings of one view and evaluated on the other. \ourmethod\ significantly outperforms the non-linear UCA baseline across all evaluated datasets. Bold values indicate the best unsupervised performance, and the paired CCA results are provided in gray at the bottom.}

\begin{adjustbox}{width=\textwidth}
\begin{tabular}{lccccc}
\toprule
Method & PC-KL & PC-PA & KL-PA & SNARE & COCO \\
\midrule
UCA (nonlinear) & \( 0.091 \pm 0.016 \) & \( 0.091 \pm 0.029 \) & \( 0.078 \pm 0.051 \) & \( 0.518 \pm 0.354 \) & \( 0.132 \pm 0.041 \) \\
\rowcolor{lightgreen}
UCCA (ours) & \( \mathbf{0.292} \pm 0.074 \) & \( \mathbf{0.315} \pm 0.025 \) & \( \mathbf{0.564} \pm 0.022 \) & \( \mathbf{0.849} \pm 0.021 \) & \( \mathbf{0.232} \pm 0.016 \) \\
\midrule
\textcolor{gray}{Paired CCA} & \textcolor{gray}{\( 0.562 \pm 0.000 \)} & \textcolor{gray}{\( 0.516 \pm 0.000 \)} & \textcolor{gray}{\( 0.629 \pm 0.000 \)} & \textcolor{gray}{\( 0.887 \pm 0.000 \)} & \textcolor{gray}{\( 0.232 \pm 0.009 \)} \\
\bottomrule
\end{tabular}
\end{adjustbox}
\label{tab:app:uca_nonlinear_classification}
\end{table}

\subsection{Anchors Choice} \label{app:anchors}

\ourmethod\ ensures scalability by adopting an anchor-based approach, a standard and well-established technique for manifold approximation \cite{williams2000using, de2004sparse}. We opt for simplicity and practicality by utilizing K-Means centroids \cite{kumar2012sampling, coates2011analysis} to construct the geometric skeleton of the latent spaces. This section provides further details on this extraction choice and the hyperparameter governing the number of anchors.

\paragraph{Anchor Extraction.} To reduce the computational burden of the QAP matching step, we perform alignment on a reduced subset of anchors rather than the entire dataset. Tab. \ref{tab:app:anchor_selection} evaluates our simple K-Means strategy against a naive random sampling baseline, alongside an idealized oracle setup (denoted as paired K-Means) designed to test the maximum capacity of the chosen centroids. Specifically: Random, randomly sampling points from each view; K-Means, the default approach for \ourmethod, where we independently run K-Means clustering on both views and extract the resulting centroids to serve as anchors; and Paired K-Means, where we run K-Means on the first view to obtain centroids, and then use the ground-truth dataset pairing to retrieve their exact corresponding points in the second view.

As demonstrated in Tab. \ref{tab:app:anchor_selection}, utilizing K-Means centroids yields a significantly more stable and representative skeleton compared to random sampling, resulting in consistently higher TC. Furthermore, the Paired K-Means oracle demonstrates that if perfect correspondence between the independently learned centroids were known, the performance would approach the paired CCA upper bound on five out of the six datasets. This confirms that simple K-Means centroids generally capture the structural information necessary for optimal alignment.

In Tab. \ref{tab:app:anchor_selection_gmm}, we verify that UCCA does not depend specifically on K-Means. We replaced K-Means anchors with Gaussian Mixture Model (GMM) anchors, and the results are comparable, showing UCCA is robust to the choice of anchor extraction method.

\paragraph{Number of anchors.} The number of anchors controls the resolution of our structural skeleton. Fig. \ref{fig:app:num_anchors} visualizes the total correlation as a function of this hyperparameter. There is an inherent trade-off: increasing the number of anchors provides a denser, more accurate representation of the underlying data manifold, which naturally improves the final alignment accuracy. However, because QAP is NP-hard, the computational complexity scales significantly with the anchor count. As illustrated in the figure, \ourmethod\ improves and then stabilizes as the number of anchors increases. This behavior allows us to select a moderate number of anchors that achieve near-optimal performance while maintaining a practical computational footprint.

\begin{table}[h]
\centering
\caption{\textbf{Anchor extraction evaluation.} Total Correlation results (mean $\pm$ std) evaluating our anchor selection approach. The default K-Means strategy consistently outperforms random sampling. The "Paired K-Means" oracle, which uses ground-truth pairings to match the centroids of the first view, demonstrates that an ideal centroid alignment approaches the paired CCA upper bound (gray). Bold values indicate the best performance.}

\begin{adjustbox}{width=\textwidth}
\begin{tabular}{lcccccc}
\toprule
Method & PC-KL & PC-PA & KL-PA & SNARE & Flickr & COCO \\
\midrule
Random & $0.840 \pm 0.426$ & $0.781 \pm 0.464$ & $0.476 \pm 1.021$ & $1.745 \pm 0.160$ & $0.878 \pm 0.217$ & $0.284 \pm 0.321$ \\
K-Means & $1.389 \pm 0.029$ & $1.388 \pm 0.026$ & $1.963 \pm 0.009$ & $1.755 \pm 0.031$ & $0.720 \pm 0.080$ & $1.369 \pm 0.116$ \\
Paired K-Means & $\mathbf{1.658} \pm 0.014$ & $\mathbf{1.676} \pm 0.026$ & $\mathbf{1.987} \pm 0.002$ & $\mathbf{1.850} \pm 0.001$ & $\mathbf{0.969} \pm 0.158$ & $\mathbf{1.409} \pm 0.209$ \\
\midrule
\textcolor{gray}{Paired CCA} & \textcolor{gray}{\(1.979 \pm 0.000\)} & \textcolor{gray}{\(1.993 \pm 0.000\)} & \textcolor{gray}{\(1.990 \pm 0.000\)} & \textcolor{gray}{\(1.850 \pm 0.000\)} & \textcolor{gray}{\(1.574 \pm 0.000\)} & \textcolor{gray}{\(1.768 \pm 0.001\)} \\
\bottomrule
\end{tabular}
\end{adjustbox}
\label{tab:app:anchor_selection}
\end{table}

\begin{table}[h]
\centering
\caption{\textbf{Additional anchor extraction methods.} Total Correlation results (mean $\pm$ std) evaluating our anchor selection approach (K-Means) and an additional GMM option. The two methods are comparable.}

\begin{adjustbox}{width=\textwidth}
\begin{tabular}{lcccccc}
\toprule
Method & PC-KL & PC-PA & KL-PA & SNARE & Flickr & COCO \\
\midrule
K-Means & $\mathbf{1.389} \pm 0.029$ & $\mathbf{1.388} \pm 0.026$ & $\mathbf{1.963} \pm 0.009$ & $\mathbf{1.755} \pm 0.031$ & $0.720 \pm 0.080$ & $1.369 \pm 0.116$ \\
GMM & $1.348 \pm 0.068$ & $1.340 \pm 0.074$ & $1.952 \pm 0.013$ & $1.685 \pm 0.111$ & $\mathbf{1.367} \pm 0.061$ & $\mathbf{1.528} \pm 0.061$ \\
\midrule
\textcolor{gray}{Paired CCA} & \textcolor{gray}{\(1.979 \pm 0.000\)} & \textcolor{gray}{\(1.993 \pm 0.000\)} & \textcolor{gray}{\(1.990 \pm 0.000\)} & \textcolor{gray}{\(1.850 \pm 0.000\)} & \textcolor{gray}{\(1.574 \pm 0.000\)} & \textcolor{gray}{\(1.768 \pm 0.001\)} \\
\bottomrule
\end{tabular}
\end{adjustbox}
\label{tab:app:anchor_selection_gmm}
\end{table}

\begin{figure}
    \centering
    \includegraphics[width=\linewidth]{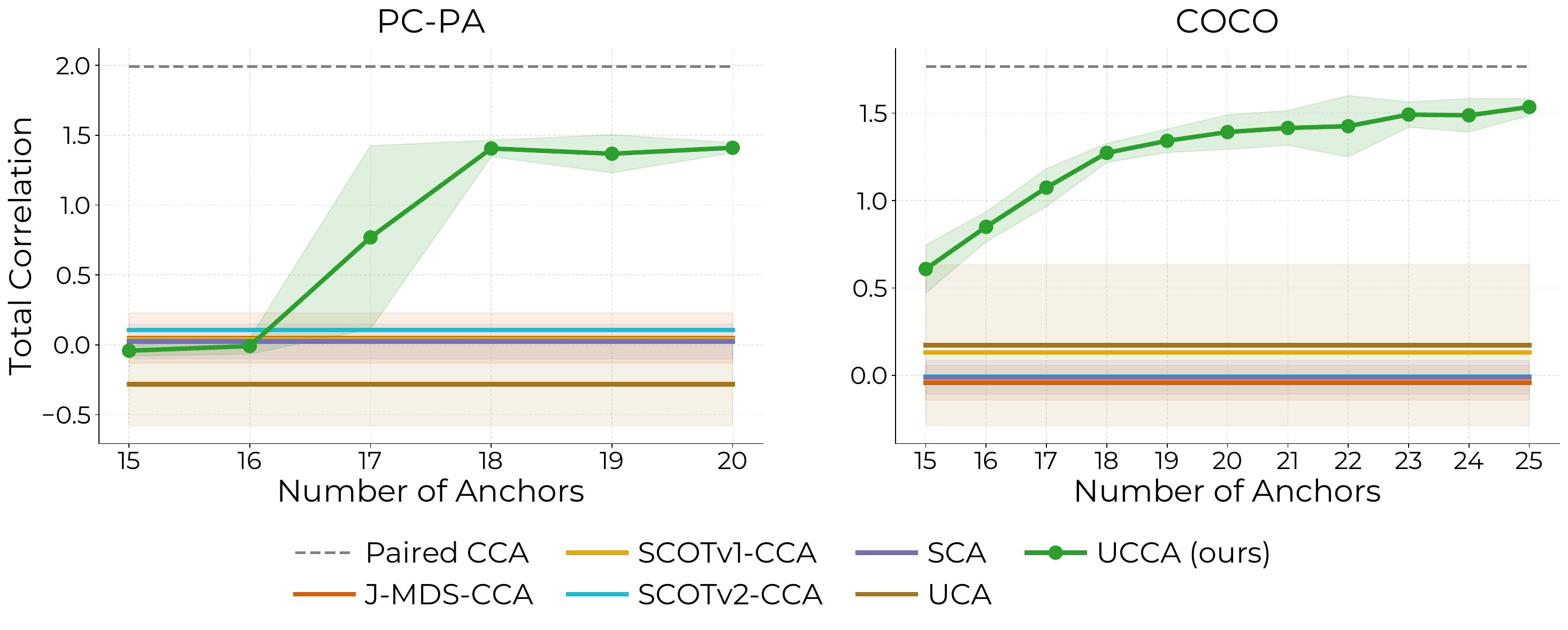}
    \caption{\textbf{Effect of the number of anchors.} Total Correlation versus the number of anchors used in the matching step. Increasing the number of anchors allows \ourmethod\ to capture the manifold geometry more accurately, leading to a sharp performance improvement that approaches the paired CCA upper bound (represented by the dashed horizontal line). The anchor count 20 is chosen to balance this high accuracy with computational efficiency.}
    \label{fig:app:num_anchors}
\end{figure}

\subsection{Effect of PCA dimension}
As detailed in App. \ref{app:implementation_details:preprocess}, to ensure computational efficiency and provide a stable, regularized feature space for the structural matching step, we apply PCA to each view independently. In Fig. \ref{fig:app:pca_dim}, we evaluate several whitening dimensions while maintaining $d\leq k$. The results show that performance is stable unless the dimension is reduced excessively (e.g., $d=5$), where informative variance is discarded.

\begin{figure}[h]
    \centering
    \includegraphics[width=\linewidth]{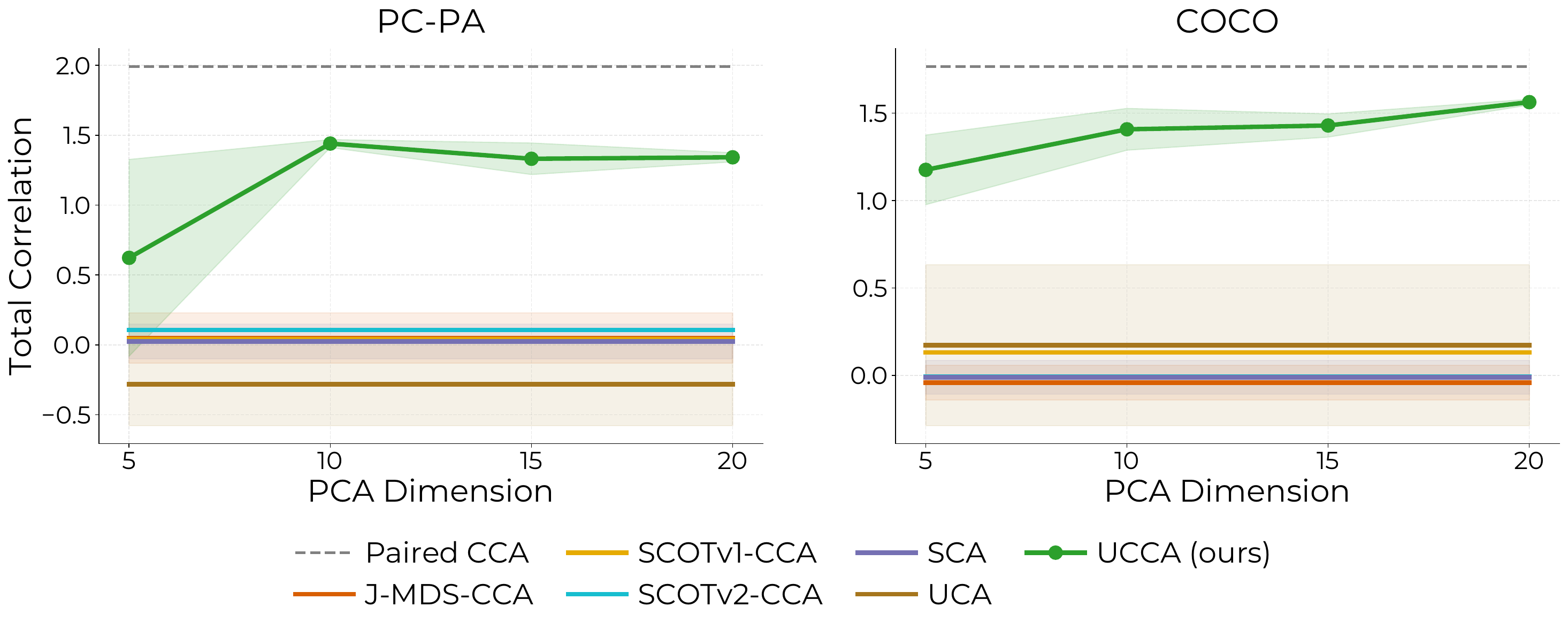}
    \caption{\textbf{Effect of the number of PCA dimensions.} Total Correlation versus the number of PCA components used in the preprocessing step.}
    \label{fig:app:pca_dim}
\end{figure}

\subsection{Visualizations}
Fig. \ref{fig:app:visus} depicts two-dimensional plots of the learned embeddings obtained by the Paired CCA oracle and UCCA. We evaluate both methods on our representative PC-PA datasets as well as the SNARE dataset, which comprises four distinct classes, facilitating straightforward visualization and interpretation of the learned representations. These plots demonstrate why UCCA achieves near-oracle classification despite a lower TC. Paired CCA forces exact instance-to-instance correspondence, maximizing TC and perfectly overlaying samples across views. In contrast, UCCA, which lacks paired supervision, introduces slight pairwise variance, marginally lowering TC. However, UCCA successfully aligns the global manifold, tightly grouping semantic clusters across modalities. Because cross-view classifiers rely on global class separability rather than perfect point-to-point mapping, UCCA's structural alignment achieves accuracy closely rivaling the oracle bound.

\begin{figure}[h]
    \centering
    \includegraphics[width=\linewidth]{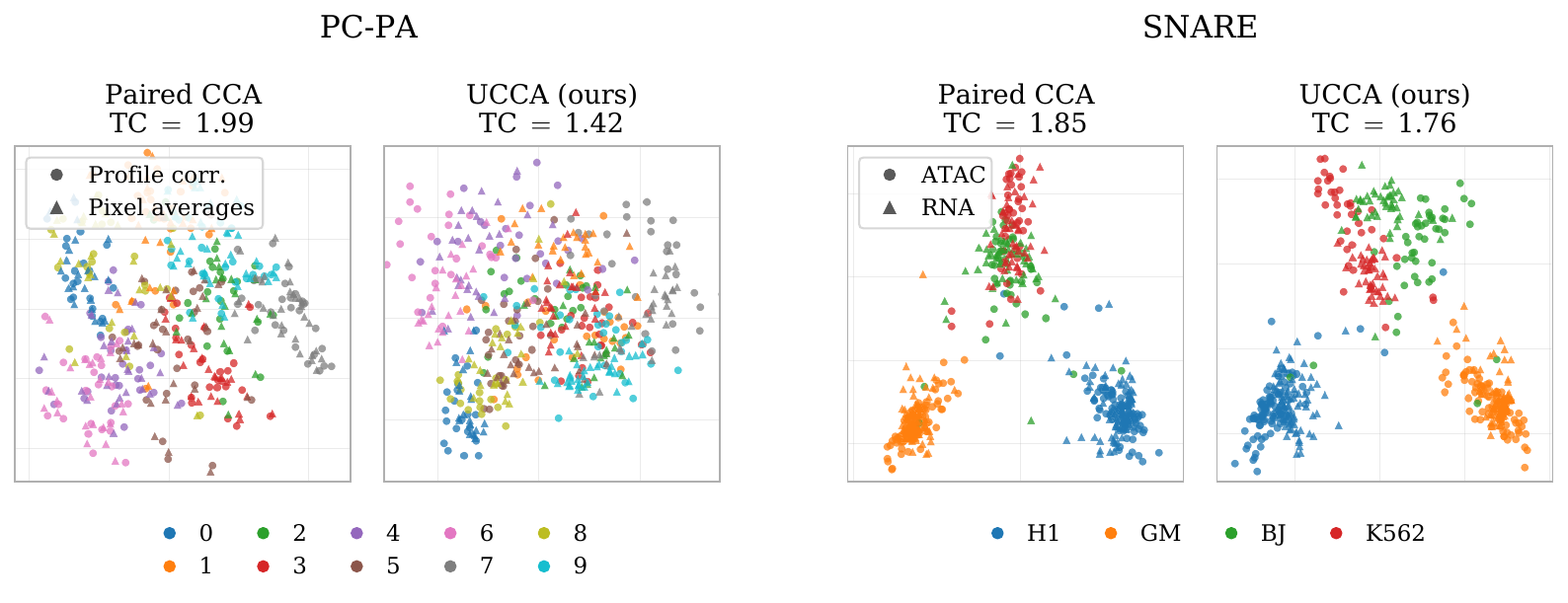}
    \caption{\textbf{Embeddings visualizations.} Visualizations of the generated embeddings of Paired-CCA and UCCA.}
    \label{fig:app:visus}
\end{figure}

\section{Additional Related Work}
This section discusses additional related work.

\paragraph{Shuffled Regression.} Both shuffled regression and UCCA tackle the fundamental challenge of learning from data where exact sample correspondences are missing. The optimization of shuffled regression has also been shown to result in a QAP in special cases \cite{lufkin2024sharp, kohjima2026gaussian}. However, while in shuffled regression $Y$ represents target labels that are known up to a permutation, in the UCCA formulation $X$ and $Y$ represent disjoint views containing shared and unique information. Shuffled regression functions as a supervised learning task aimed at predicting target responses or recovering regression coefficients from unmatched inputs. UCCA addresses unsupervised multiview representation learning, focusing on extracting a shared space in scenarios where paired correspondences are entirely unavailable rather than merely permuted \cite{xu2019uncoupled, slawski2024permuted, lufkin2024sharp, kohjima2026gaussian}. Furthermore, shuffled regression models often rely on additional assumptions, such as assuming $Y$ is a scalar \cite{xu2019uncoupled, lufkin2024sharp, kohjima2026gaussian}, imposing strong structural constraints (e.g., cyclical monotonicity \cite{slawski2024permuted}), or requiring auxiliary weak supervision (e.g., pairwise comparison data \cite{xu2019uncoupled}). UCCA instead mitigates the complexity of the QAP through an anchor-based geometric alignment procedure.

\paragraph{Canonical Time Warping (CTW).} CTW combines CCA with Dynamic Time Warping (DTW) to achieve spatio-temporal alignment \cite{zhou2009canonical}. However, temporal alignment is simpler than optimizing unconstrained permutations because strict sequence constraints (boundary conditions, continuity, and monotonicity) enable efficient dynamic programming. Conversely, UCCA addresses completely generic data, necessitating the computationally heavier QAP. Nevertheless, incorporating structural constraints such as temporal ordering into UCCA is an interesting future direction.

\section{Complexity Analysis} \label{app:complexity}
This section analyzes the time complexity of \ourmethod. Recall the notations: \(n\), the number of samples in the input datasets (say the largest between the two); \(d\), the number of dimensions; \(k\), the number of anchors. Also denote: \(t\), the number of iterations for the K-Means clustering algorithm; \(c\), the number of clustering repetitions; \(I\), the number of iterations for the 2-opt solver.

We start by analyzing each step in Alg. \ref{alg:method}. First is the whitening step. This takes \(O(nd^2)\) to compute the empirical covariance matrix, and \(O(d^3)\) its SVD, summing up to \(O(nd^2 + d^3)\). The second step is the anchors extraction, specifically K-Means. This is known to have time complexity of \(O(tnkd)\). We repeat the K-Means process \(c\) times.

For the QAP step, we use the 2-opt approximate solver. The similarity matrices take \(O(k^2d)\). The 2-opt solver works at \(O(Ik^3)\), making the complexity of the third step \(O(k^2d + Ik^3)\). This step is repeated \(c\) times, one for each clustering. Finally, we apply CCA to these concatenated anchors, which takes \(O(ckd^2 + d^3)\).

The overall complexity is then,
\[O(nd^2 + ctnkd + cIk^3 + ckd^2 + d^3)\]

Consider that for a large \(n\), we generally consider \(k, d \ll n\), and so as the number of iterations and clusterings are constants. Then, \ourmethod\ is linear in the number of samples \(n\).

\section{Additional Details on Experiments}
\subsection{QAP Alignment with CCA} \label{app:qap_cca}

This section provides the experimental details corresponding to Fig. \ref{fig:open}a, which empirically demonstrates the core theoretical insight of \ourmethod: the direct relationship between the QAP objective and CCA performance.

\paragraph{Synthetic Dataset.} To exhaustively evaluate all possible pairings between two views, we constructed a simple synthetic toy dataset, visualized in Fig. \ref{fig:app:toy_data}. The dataset consists of \(n=10\) samples in two distinct views, structured as unaligned but geometrically similar 1-dimensional manifolds in a 2-dimensional space (resembling hooks).

\begin{wrapfigure}{r}{0.3\textwidth} 
    \vspace{-20pt}
    \centering
    \includegraphics[width=0.3\textwidth]{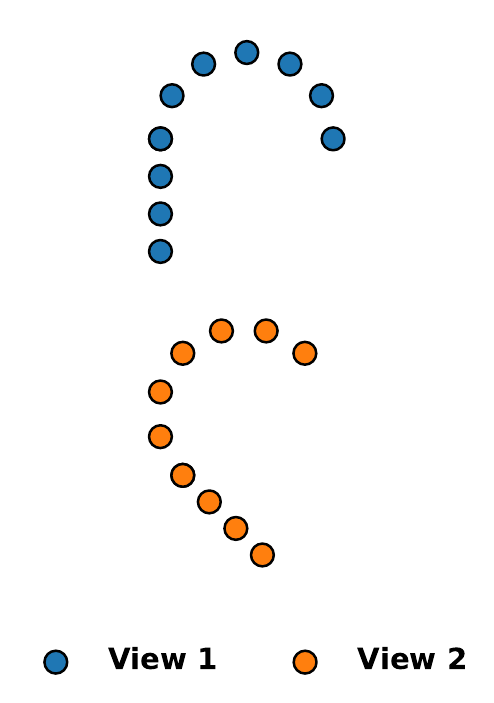}
    \vspace{-10pt}
    \caption{\textbf{Synthetic toy dataset.}}
    \label{fig:app:toy_data}
    \vspace{-20pt}
\end{wrapfigure}

\paragraph{Exhaustive Evaluation.} Because the dataset contains exactly 10 samples per view, the set of all possible permutations is finite and computationally tractable to explore in its entirety (\(10! \approx 3.6 \times 10^6\) permutations). For every possible permutation matrix \(P\) in this space, we performed the following two steps: (1) QAP Value: We computed the QAP objective value using a linear kernel; (2) Total Correlation: We used $P$ to pair the samples from View 1 and View 2, fitted a paired CCA on this specific pairing, and recorded the resulting TC of the true pairing under the learned embeddings.

\paragraph{Results and Discussion.} The density plot in Fig. \ref{fig:open}a visualizes the joint distribution of these \(\sim 3.6\) million pairs of (QAP value, TC). The empirical results reveal a strict, positive relationship between the two metrics: as the QAP value increases, the structural alignment improves, leading to a higher total correlation.

Crucially, the permutation that yields the absolute maximum QAP value (denoted by the red star in Fig. \ref{fig:open}a) aligns almost perfectly with the permutation that achieves the maximum possible total correlation (the yellow triangle). This exhaustive, brute-force experiment provides direct empirical demonstration of our key theoretical insight: solving the QAP with a linear kernel effectively acts as a proxy for maximizing the CCA objective in an unpaired setting.

\subsection{Assumption \ref{assum:weak_majorization} Validation} \label{app:assum_valid}

\begin{figure}
    \centering
    \includegraphics[width=\linewidth]{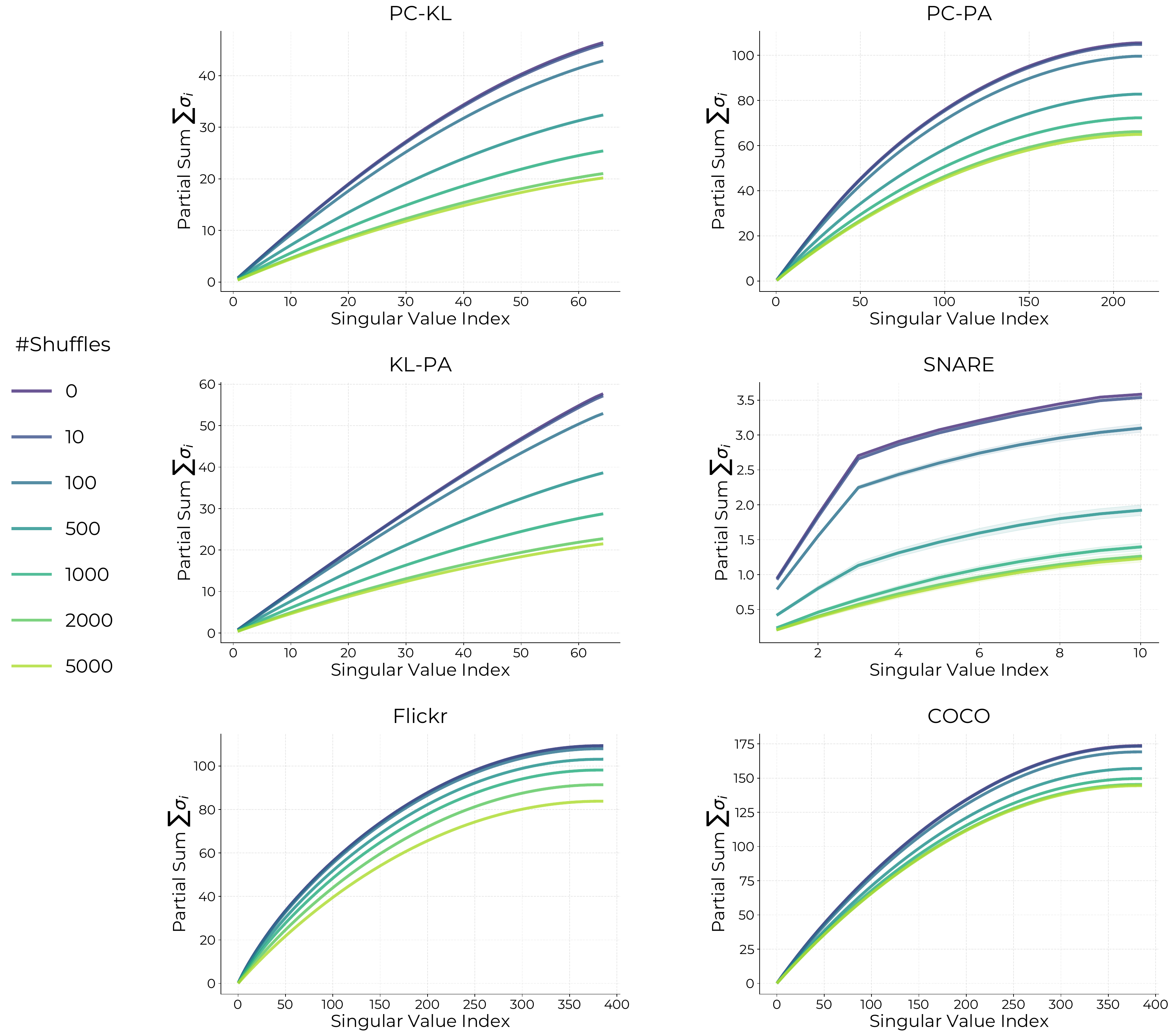}
    \caption{\textbf{Extended empirical validation of Assum. \ref{assum:weak_majorization}.} Partial cumulative sums of singular values for the true permutation (\(P^*\), denoted by 0 shuffles) and various randomly shuffled permutations across all datasets: PC-KL, PC-PA, KL-PA, SNARE, Flickr, and COCO. Consistent with the main text results (Fig. \ref{fig:theory}a), the partial sums strictly decrease as the permutation diverges from \(P^*\), verifying the weak majorization property across all tasks.}
    \label{fig:app:assum_valid}
\end{figure}

This section provides further empirical evidence supporting Assum. \ref{assum:weak_majorization}, which posits that the singular values under the optimal alignment weakly majorize those under any alternative permutation. In the main text, Fig. \ref{fig:theory}a demonstrates this property on the Flickr dataset by showing that the partial cumulative sums of singular values strictly decrease as the permutation increasingly diverges from the ground truth.

Fig. \ref{fig:app:assum_valid} extends this validation to the remaining benchmark datasets: PC-KL, PC-PA, KL-PA, SNARE and COCO. Following the same procedure, we compute the partial sums of singular values for the true permutation (\(P^*\), represented by 0 shuffles) and for permutations subjected to an increasing number of random shuffles. Across all evaluated domains, the cumulative sums are consistently lower as the number of shuffles increases. This universal trend strongly corroborates the weak majorization behavior, confirming it is a robust and practically valid assumption for real-world datasets.

\subsection{Theorem \ref{thm:mcp_equiv} Validation} \label{app:thm_valid}

\begin{figure}
    \centering
    \includegraphics[width=\linewidth]{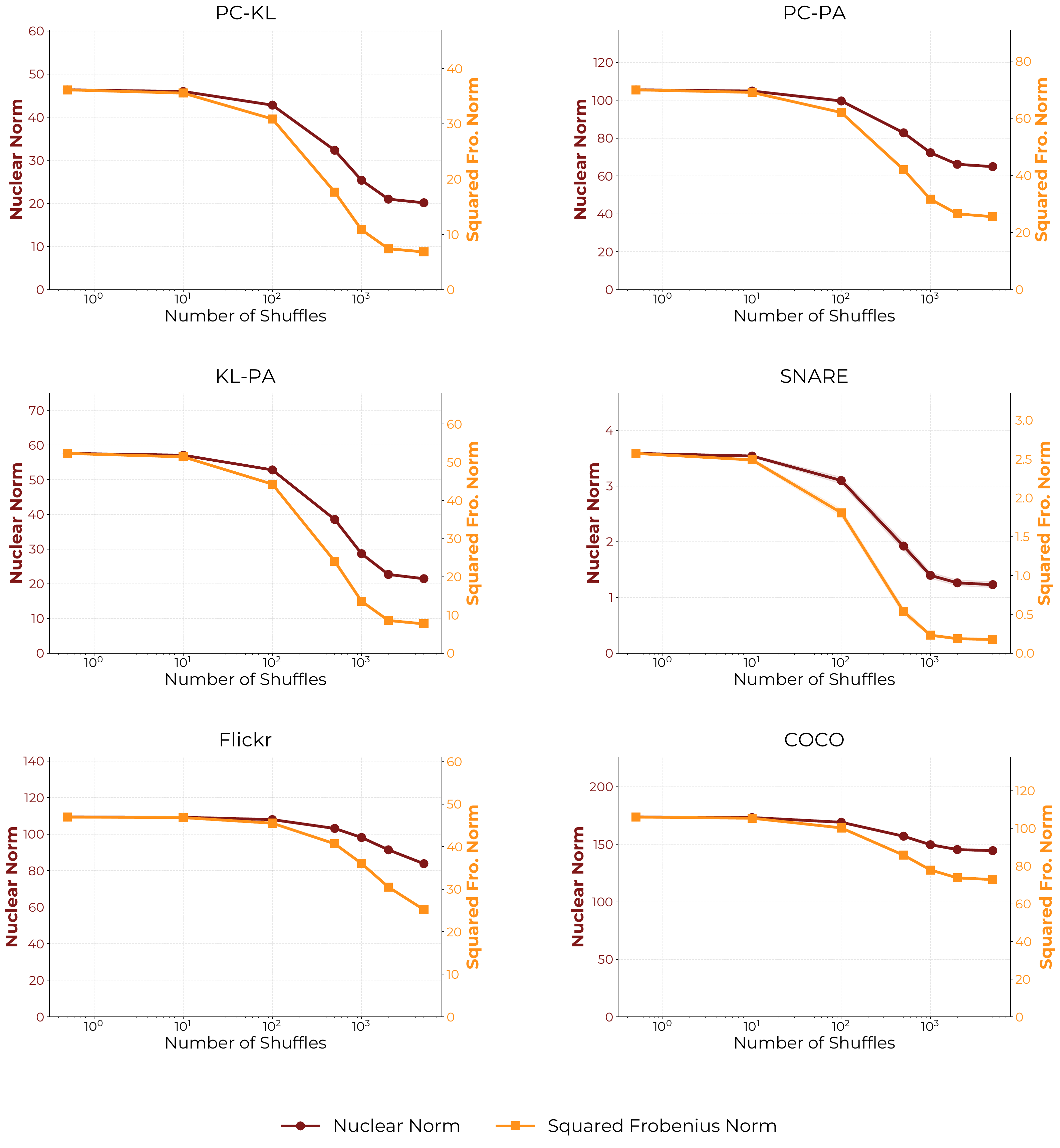}
    \caption{\textbf{Extended empirical validation of Theorem 2.} Comparison of the \(\mcp_{O(d)}\) (Nuclear Norm) and \(\mcp_{S_F}\) (Squared Frobenius Norm) objectives as the permutation is randomized away from the true alignment \(P^*\) across all datasets: PC-KL, PC-PA, KL-PA, SNARE, Flickr, and COCO. In alignment with the main text results (Fig. \ref{fig:theory}b), both objectives decay synchronously as the number of random shuffles increases, empirically confirming that both formulations are locally maximized by the same ground-truth permutation.}
    \label{fig:app:thm_valid}
\end{figure}

This section provides supplementary results for the empirical validation of Thm. \ref{thm:mcp_equiv}, which establishes that under Assum. \ref{assum:weak_majorization}, \(\mcp_{O(d)}\) and the \(\mcp_{S_F}\) coincide. In the main text, Fig. \ref{fig:theory}b illustrates this property on the Flickr dataset by tracking the respective objectives, the Nuclear Norm for \(\mcp_{O(d)}\) and the Squared Frobenius Norm for \(\mcp_{S_F}\), as the permutation increasingly diverges from the ground-truth alignment \(P^*\).

Fig. \ref{fig:app:thm_valid} extends this comparative analysis to the PC-KL, PC-PA, KL-PA, SNARE, and COCO datasets. Across all evaluated domains, we observe a consistent, synchronized trend: both the Nuclear Norm and the Squared Frobenius Norm strictly decay as the number of random shuffles increases. This shared decay profile demonstrates that both objectives are locally maximized by the exact same ground-truth permutation, further corroborating the theoretical equivalence established in Thm. 2.

\section{Implementation Details} \label{app:implementation_details}

This section provides a comprehensive description of the implementation details for \ourmethod, alongside the exact configurations used for all evaluated baselines.

\subsection{Preprocessing and Dimensionality Reduction} \label{app:implementation_details:preprocess}
In alignment with Alg. 1, the raw input data for each view is first subjected to mean-centering and whitening. Let \(X_{raw} \in \mathbb{R}^{n \times d_1}\) and \(Y_{raw} \in \mathbb{R}^{n \times d_2}\) be the raw training sample matrices for the first and second views, respectively. We compute the empirical training means and center both the training and test sets.

To ensure computational efficiency and provide a stable, regularized feature space for the structural matching step, we apply Principal Component Analysis (PCA) \cite{pearson1901liii} to each view independently. By default, we reduce the dimensionality of both views to \(d_{pca}=10\) components and whiten the representations. This process yields the whitened, dimensionally reduced latent spaces \(Z_X, Z_Y \in \mathbb{R}^{n \times 10}\), completing the \(Z_X, Z_Y \leftarrow \text{Whiten}(X, Y)\) step of our algorithmic framework.

\subsection{Unsupervised Anchor Selection}
To extract the geometric skeleton of the manifolds, we rely on K-Means clustering to select representative anchors for each view (\(A_X, A_Y \leftarrow \text{Anchors}(Z_X, Z_Y)\)). Given the whitened views \(Z_X\) and \(Z_Y\), we independently run K-Means to extract \(k=20\) cluster centroids per view. For the SNARE datasets we lowered the number of anchors to \(k=12\) due to the lower data amount. The resulting cluster centroids serve as our anchors \(A_X \in \mathbb{R}^{k \times d_{pca}}\) and \(A_Y \in \mathbb{R}^{k \times d_{pca}}\) for the structural alignment.

\subsection{Anchor Matching via QAP}
We align the geometric skeletons by solving the linear-kernel QAP. We construct linear similarity kernels for both sets of anchors:
\begin{equation*}
K_X = A_X A_X^T, \quad K_Y = A_Y A_Y^T
\end{equation*}
We then search for the permutation matrix \(P'\) that maximizes the QAP structural alignment between the kernels:
\begin{equation*}
P' = \arg\max_{P \in \mathcal{P}} \text{tr}(K_X P K_Y P^T)
\end{equation*}
This maximization is performed using the 2-Opt heuristic solver available in SciPy's quadratic\_assignment module \cite{2020SciPy-NMeth}. Because the QAP is NP-hard and 2-Opt is a local search heuristic, we perform \(N_{restart}=50\) random restarts and select the permutation that yields the highest objective value.

\subsection{Learning the Projections}
Once the structural permutation \(P'\) is estimated, we pair the anchors to form pseudo-pairs \((A_X, P' A_Y)\). To ensure robustness against any single suboptimal clustering or matching instance, we repeat the entire clustering and alignment process \(N_{trials}=500\) times, concatenating all resulting pseudo-paired anchors.

Finally, we apply a standard paired CCA directly to these concatenated pseudo-pairs to learn the final linear projection matrices \(U\) and \(V\). These learned weights are then utilized to project the complete, original sets of high-dimensional samples into the shared maximally correlated latent space.

\newpage
\subsection{Data Split and Unpaired Setting}

\begin{figure}[h]
    \centering
    \includegraphics[width=0.8\linewidth]{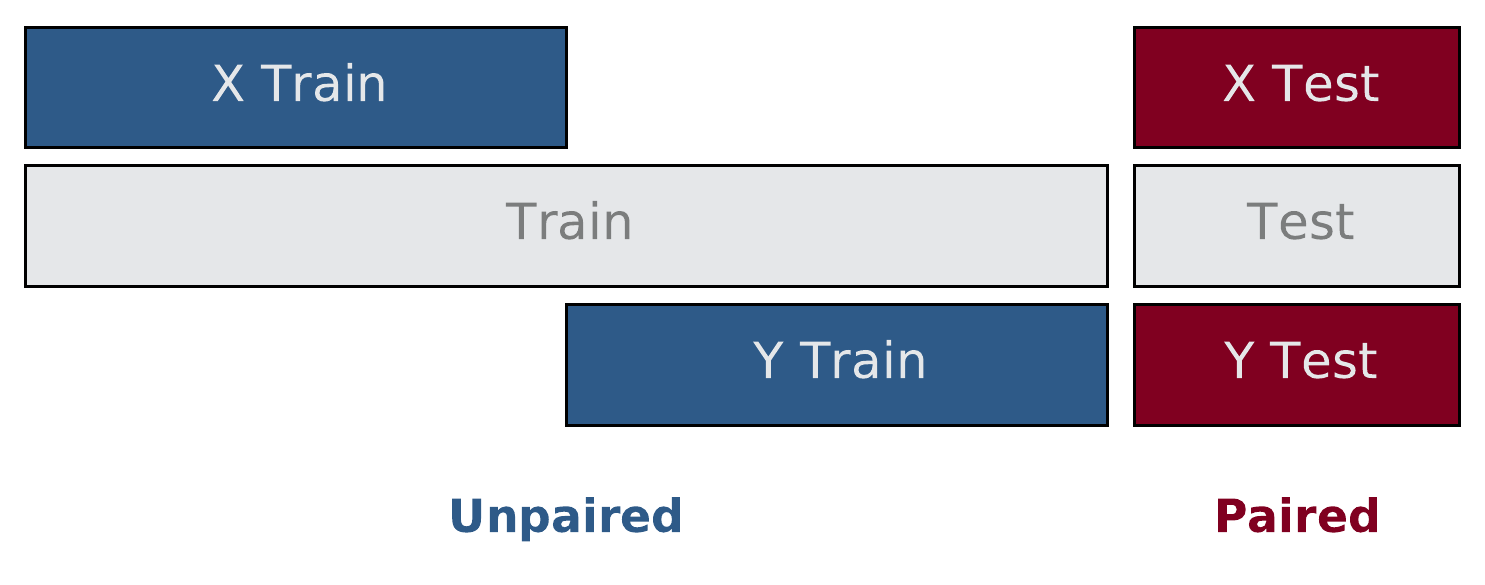}
    \caption{\textbf{Data preparation:} While the test is kept paired for evaluation purposes, we split the train such that no correspondence is available during training. This exactly replicates the unpaired settings discussed in the paper.}
    \label{fig:app:data_prep}
\end{figure}

To rigorously evaluate the method in a strictly unpaired setting, we follow a specific dataset splitting protocol (visualized in Fig. \ref{fig:app:data_prep}). For a dataset with \(N\) total paired samples:
\begin{enumerate}
    \item A common paired test set \(\mathcal{D}_{test}\) is held out (typically \(N_{test}=256\) samples).
    \item The remaining \(N - N_{test}\) training samples are divided into two disjoint subsets, \(S_1\) and \(S_2\), of size \((N - N_{test}) / 2\).
    \item The training set for View 1, \(X_{train}\), is exclusively formed by taking the first view of the samples in \(S_1\).
    \item The training set for View 2, \(Y_{train}\), is exclusively formed by taking the second view of the samples in \(S_2\).
\end{enumerate}
This protocol guarantees that there is strictly zero sample correspondence between \(X_{train}\) and \(Y_{train}\) during the training phase (i.e., for any sample index \(i\) present in the first view, its corresponding second view is actively withheld from the training data).

\subsection{Extended Algorithm}
Alg. \ref{alg:app:extended} presents a detailed summary of \ourmethod, presented in Alg. \ref{alg:method}. The primary differences are the inclusion of specific hyperparameters and precise implementations. Specifically, the pre-processing step is defined as PCA whitening, and K-Means clustering is utilized to extract the anchor sets. Additionally, the QAP matching step is approximated using a 2-opt solver, which is executed with \(r\) random restarts to escape local optima and capture the maximum objective value. Finally, the extended algorithm explicitly details the matching loop: the clustering and alignment process is repeated \(c\) times, and the resulting matched anchors are concatenated to form a larger, comprehensive reference set for the final CCA projection.

\begin{algorithm}[H]
  \caption{Unpaired CCA (Extended)}
  \label{alg:app:extended}
  \begin{algorithmic}[1]
    \Require Unpaired data $X \in \mathbb{R}^{n \times d_x}$, $Y \in \mathbb{R}^{n \times d_y}$
    \Require $k$ (number of centroids), $c$ (clustering repetitions), $r$ (2-opt restarts)
    \Ensure Projectors $U, V$
    
    \State $Z_X \gets \text{PCA-Whiten}(X)$
    \State $Z_Y \gets \text{PCA-Whiten}(Y)$
    \State Initialize concatenated anchor sets $\tilde{A}_X \gets \emptyset$, $\tilde{A}_Y \gets \emptyset$
    
    \For{$i = 1 \dots c$}
        \State $A_X^{(i)} \gets \text{K-Means}(Z_X, k)$
        \State $A_Y^{(i)} \gets \text{K-Means}(Z_Y, k)$
        
        \State $o_{\text{max}} \gets -\infty$, $P' \gets I$
        \For{$j = 1 \dots r$}
            \State $P \gets \text{2opt-QAP}(A_X^{(i)}, A_Y^{(i)})$
            \State $o \gets \text{QAP-Objective}(P, A_X^{(i)}, A_Y^{(i)})$
            \If{$o > o_{\text{max}}$}
                \State $o_{\text{max}} \gets o$
                \State $P' \gets P$
            \EndIf
        \EndFor
        
        \State $\tilde{A}_X \gets \text{Concat}(\tilde{A}_X, A_X^{(i)})$
        \State $\tilde{A}_Y \gets \text{Concat}(\tilde{A}_Y, P' A_Y^{(i)})$
    \EndFor
    
    \State $U, V \gets \text{CCA}(\tilde{A}_X, \tilde{A}_Y)$ \Comment{Maximize correlation relative to the matched anchors}
    \State \Return $U, V$
  \end{algorithmic}
\end{algorithm}

\subsection{Datasets}
We evaluate UCCA on a diverse set of datasets spanning multiple domains. Tab. \ref{tab:app:datasets} summarizes their properties:

\begin{table}[h]
\small
\centering
\caption{Properties of datasets used in evaluations.}
\begin{adjustbox}{width=\textwidth}
\begin{tabular}{lcccc}
\toprule
Dataset & Views & Samples ($N$) & Dimensions ($d_1, d_2, \dots$) & Classes ($C$) \\
\midrule
Handwritten & PC, KL, PA & 2,000 & 216, 64, 240 & 10 \\
SNARE & ATAC, RNA & 1,047 & 19, 10 & 4 \\
Flickr & Image, Text & 8,000 & 384, 384 & - \\
COCO & Image, Text & 2,000 & 384, 384 & 79 \\
\bottomrule
\end{tabular}
\end{adjustbox}
\label{tab:app:datasets}
\end{table}

\textbf{UCI Handwritten Digits.} Consists of 2,000 samples of digits 0-9. We use three continuous views: profile correlations (PC), Karhunen-Loève coefficients (KL), and pixel averages (PA). These provide varied feature spaces for cross-view alignment evaluation. We reserved a test set of 256 samples, leaving 872 samples from each view for training.

\textbf{SNARE.} A single-cell dataset generated by the SNARE-seq assay, containing co-measured chromatin accessibility (ATAC-seq) and gene expression (RNA-seq) for 1,047 cells. We utilize the preprocessed embeddings provided by Chen et al. \cite{chenunsupervised}, where the ATAC and RNA views are dimensionally reduced to 19 and 10 dimensions, respectively, and unit-normalized. The dataset is categorized into 4 distinct cell type classes (GM, BJ, K562, and H1), representing a challenging, real-world biological alignment task. We reserved a test set of 257 samples, leaving 395 samples from each view for training.

\textbf{Flickr.} Contains image-text pairs. We utilized the 8,000 pairs training set from jxie/flickr8k\footnote{\url{https://huggingface.co/datasets/jxie/flickr8k}}. Image features are extracted using the pre-trained DINOv2-ViT-S/14 model. Text features are generated by processing the primary caption (caption\_0) through the all-MiniLM-L6-v2 sentence transformer using mean pooling. Both the visual and textual feature pipelines yield 384-dimensional dense representations. We reserved a test set of 512 samples, leaving 2744 samples from each view for training.

\textbf{COCO.} We used the 2,000 image-text pairs from the COCO 2017 validation subset, sourced from the riddhimanrana/coco-fastvlm-2k-val2017 dataset\footnote{\url{https://huggingface.co/datasets/riddhimanrana/coco-fastvlm-2k-val2017}}. Similar to Flickr, the two views represent image and text modalities. The image view is encoded using DINOv2-ViT-S/14 into a 384-dimensional space. The text view is generated from GPT-structured conversational captions, encoded via the all-MiniLM-L6-v2 model into 384 dimensions. Unique object-category labels are extracted from the human prompts in the conversational turns to establish discrete classes for evaluation, following the object detection labels from the dataset. For the paired CCA upper bound, we reduced each view independently to its top 10 principal components, as it improved performance. We reserved a test set of 256 samples, leaving 872 samples from each view for training.

\subsection{Experiment-Specific Details}

\textbf{Cross-view classification}: For the downstream cross-view classification tasks, we evaluate the learned embeddings using a \(k\)-Nearest Neighbors (kNN) classifier configured with \(k=5\) neighbors. The results in Tab. \ref{tab:classification} show the classifier fitted on the following view of each dataset: PC-KL: PC, PC-PA: PC, KL-PA: KL, SNARE: ATAC, COCO: text. Correspondingly, Tab. \ref{tab:app:cross_class_knn21} shows the results of the classifier trained on the alternate view of each dataset: PC-KL: KL, PC-PA: PA, KL-PA: PA, SNARE: RNA, COCO: image.

\textbf{Runtime vs. Performance}: The results presented in Fig. \ref{fig:runtime} and Fig. \ref{fig:app:runtime} were evaluated across the following set of clustering repetitions: \(\{1, 5, 10, 20, 30, 50, 100, 200, 500\}\).

\textbf{Amount of Unpaired data}: We evaluate the robustness of \ourmethod\ by varying the amount of unpaired training data from 50\% to 100\% of the full available training set. 

\textbf{Anchor Choice:} The Paired K-Means results in Tab. \ref{tab:app:anchor_selection} utilize a K-Means model trained on a single view. For each centroid, we identify the closest sample from the training set and retrieve its corresponding pair in the alternate view. Note that these paired samples are not present in the standard unpaired training split used throughout the rest of the paper. The K-Means model was trained on the following views for each dataset: PC for PC-KL and PC-PA, KL for KL-PA, ATAC for SNARE, and text for both Flickr and COCO.

\subsection{Baseline Implementations}

Since direct baselines for unpaired CCA do not currently exist, we adapted the most closely related unpaired alignment methods for our benchmark. To ensure optimal baseline performance, we performed a hyper-parameter search for each method. Given the inherent difficulty of hyperparameter selection in a fully unpaired setting, we tuned the baseline hyperparameters on the test set of the PC-PA views from the Handwritten benchmark. These parameters were subsequently frozen and applied across all other datasets. Detailed descriptions of each baseline are provided below.

\paragraph{SCA \cite{timilsina2024identifiable}.} We utilized the official implementation provided by the authors from \url{https://github.com/XiaoFuLab/Shared-Component-Analysis}. The hyperparameter grid consisted of the following values: \texttt{orthogonal\_w} \(\in \{0.1, 1.0, 10.0\}\), \texttt{encoder\_lr} \(\in \{0.001, 0.01\}\), and \texttt{discr\_lr} \(\in \{0.0001, 0.001\}\). The optimal hyperparameters yielded by this search were: \texttt{orthogonal\_w} \(= 1.0\), \texttt{encoder\_lr} \(= 0.01\), and \texttt{discr\_lr} \(= 0.0001\).

\paragraph{UCA \cite{hoshen2018unsupervised}.} As no official implementation is provided by the authors, we self implemented the method based on the paper. We report the performance of UCA's linear formulation in the main text to maintain an architecturally fair comparison to \ourmethod, providing the non-linear neural network version in App. \ref{app:uca_nonlinear}. To maintain computational efficiency, training sets exceeding 5,000 samples were randomly subsampled to 5,000 instances.

For the linear version of UCA, the hyperparameter grid consisted of the following values: \texttt{lr} \(\in \{0.0001, 0.001\}\), \texttt{lambda\_rec} \(\in \{1.0, 10.0\}\), and \texttt{lambda\_cyc} \(\in \{1.0, 10.0\}\). The optimal hyperparameters yielded by this search were: \texttt{lr} \(= 0.001\), \texttt{lambda\_rec} \(= 10.0\), and \texttt{lambda\_cyc} \(= 10.0\).

For the nonlinear version, we employed a simple MLP architecture with three hidden layers of sizes \([512, 256, 10]\) for the encoder, and reversed dimensions of \([256, 512, 2]\) for the decoder. The hyperparameter grid consisted of the following values: \texttt{lr} \(\in \{0.0001, 0.001\}\), \texttt{lambda\_rec} \(\in \{1.0, 10.0\}\), and \texttt{lambda\_cyc} \(\in \{1.0, 10.0\}\). The optimal hyperparameters yielded by this search were: \texttt{lr} \(= 0.001\), \texttt{lambda\_rec} \(= 1.0\), and \texttt{lambda\_cyc} \(= 10.0\).

\paragraph{J-MDS \cite{chenunsupervised}.} We used the official implementation in \url{https://github.com/BorgwardtLab/JointMDS}. To maintain computational efficiency, training sets exceeding 2,000 samples were randomly subsampled to 2,000 instances. The hyperparameter grid consisted of the following values: \texttt{alpha} \(\in \{0.5, 1.0, 2.0\}\), and \texttt{eps} \(\in \{0.005, 0.01, 0.05\}\). The optimal hyperparameters yielded by this search were: \texttt{alpha} \(= 2.0\), and \texttt{eps} \(= 0.05\).

\paragraph{SCOTv1 \cite{demetci2022scot} and SCOTv2 \cite{demetci2022scotv2}.} We used the official implementation in \url{https://github.com/rsinghlab/SCOT}. The hyperparameter grid consisted of the following values: \texttt{k} \(\in \{20, 50, 100\}\), and \texttt{eps} \(\in \{0.001, 0.005, 0.01, 0.05\}\). The optimal hyperparameters yielded by this search for SCOTv1 were: \texttt{k} \(= 50\), and \texttt{eps} \(= 0.01\); and for SCOTv2 were: \texttt{k} \(= 100\), and \texttt{eps} \(= 0.005\).

For J-MDS, Scotv1 and Scotv2, because these methods compute point-cloud alignments without establishing generalizable functional mappings, we adapt them to the CCA setting (denoted as the -CCA variants). Specifically, we extract the cross-view pseudo-pairs generated by their respective optimal transport or multidimensional scaling algorithms at the end of training, and subsequently fit a standard CCA model to these pseudo-pairs (matching the final step of \ourmethod).

\subsection{Hardware and Software Infrastructure} \label{app:hardware}
All experiments were implemented in Python. Core numerical and machine learning operations were executed using NumPy and scikit-learn \cite{scikit-learn} (for K-Means, PCA, and standard CCA computations). The QAP approximations were solved using the scipy.optimize.quadratic\_assignment module. The UCA baselines were trained using PyTorch.

All timing analyses and performance benchmarks were executed on a localized server equipped with one NVIDIA GeForce RTX 2080 Ti GPU and eight Intel(R) Xeon(R) Gold 6138 CPU @ 2.00GHz CPU cores. Due to the efficient dimensionality reduction and the use of the QAP solver strictly on a small subset of anchors, \ourmethod\ maintains a minimal memory and compute footprint, comfortably running on standard consumer-grade hardware.


\newpage
\section*{NeurIPS Paper Checklist}

\begin{enumerate}

\item {\bf Claims}
    \item[] Question: Do the main claims made in the abstract and introduction accurately reflect the paper's contributions and scope?
    \item[] Answer: \answerYes{} 
    \item[] Justification: All theoretical and empirical claims made in the introduction are explicitly addressed and substantiated in Sec. \ref{sec:theory} and \ref{sec:results}.

\item {\bf Limitations}
    \item[] Question: Does the paper discuss the limitations of the work performed by the authors?
    \item[] Answer: \answerYes{} 
    \item[] Justification: Limitations are discussed in Sec. \ref{sec:conclusions}.

\item {\bf Theory assumptions and proofs}
    \item[] Question: For each theoretical result, does the paper provide the full set of assumptions and a complete (and correct) proof?
    \item[] Answer: \answerYes{} 
    \item[] Justification: Assumptions and main theorems are stated in Sec. \ref{sec:theory}. Full derivations and proofs are provided in App. \ref{app:proofs}.

    \item {\bf Experimental result reproducibility}
    \item[] Question: Does the paper fully disclose all the information needed to reproduce the main experimental results of the paper to the extent that it affects the main claims and/or conclusions of the paper (regardless of whether the code and data are provided or not)?
    \item[] Answer: \answerYes{} 
    \item[] Justification: Comprehensive implementation details, data splitting protocols, and hyperparameters are provided in App. \ref{app:implementation_details}.

\item {\bf Open access to data and code}
    \item[] Question: Does the paper provide open access to the data and code, with sufficient instructions to faithfully reproduce the main experimental results, as described in supplemental material?
    \item[] Answer: \answerYes{} 
    \item[] Justification: All datasets are publicly available (App. F.7), and the code is released at \url{github.com/shaham-lab/UCCA}.

\item {\bf Experimental setting/details}
    \item[] Question: Does the paper specify all the training and test details (e.g., data splits, hyperparameters, how they were chosen, type of optimizer) necessary to understand the results?
    \item[] Answer: \answerYes{} 
    \item[] Justification: All relevant experimental settings, including strict unpaired data splits and baseline hyperparameter grids, are detailed in Sec. \ref{sec:results} and App. \ref{app:implementation_details}.

\item {\bf Experiment statistical significance}
    \item[] Question: Does the paper report error bars suitably and correctly defined or other appropriate information about the statistical significance of the experiments?
    \item[] Answer: \answerYes{} 
    \item[] Justification: All tables report means and standard deviations, and all plots visualize variance using error bars or shaded bands.

\item {\bf Experiments compute resources}
    \item[] Question: For each experiment, does the paper provide sufficient information on the computer resources (type of compute workers, memory, time of execution) needed to reproduce the experiments?
    \item[] Answer: \answerYes{} 
    \item[] Justification: Hardware specifications, runtime analyses, and computational complexity are detailed in App. \ref{app:hardware}, App. \ref{app:runtime}, and App. \ref{app:complexity}.
    
\item {\bf Code of ethics}
    \item[] Question: Does the research conducted in the paper conform, in every respect, with the NeurIPS Code of Ethics \url{https://neurips.cc/public/EthicsGuidelines}?
    \item[] Answer: \answerYes{} 
    \item[] Justification: The research involves fundamental mathematical optimization on public datasets, adhering fully to the NeurIPS Code of Ethics.

\item {\bf Broader impacts}
    \item[] Question: Does the paper discuss both potential positive societal impacts and negative societal impacts of the work performed?
    \item[] Answer: \answerNA{} 
    \item[] Justification: This is foundational mathematical research on canonical correlation analysis without foreseeable direct negative societal impacts.
    
\item {\bf Safeguards}
    \item[] Question: Does the paper describe safeguards that have been put in place for responsible release of data or models that have a high risk for misuse (e.g., pre-trained language models, image generators, or scraped datasets)?
    \item[] Answer: \answerNA{} 
    \item[] Justification: The paper does not introduce or release high-risk models or datasets.

\item {\bf Licenses for existing assets}
    \item[] Question: Are the creators or original owners of assets (e.g., code, data, models), used in the paper, properly credited and are the license and terms of use explicitly mentioned and properly respected?
    \item[] Answer: \answerYes{} 
    \item[] Justification: All utilized public datasets and official baseline implementations are properly cited in Sec. \ref{sec:results} and App. \ref{app:implementation_details}.

\item {\bf New assets}
    \item[] Question: Are new assets introduced in the paper well documented and is the documentation provided alongside the assets?
    \item[] Answer: \answerNA{} 
    \item[] Justification: The paper uses existing benchmarks and does not release new datasets or models.

\item {\bf Crowdsourcing and research with human subjects}
    \item[] Question: For crowdsourcing experiments and research with human subjects, does the paper include the full text of instructions given to participants and screenshots, if applicable, as well as details about compensation (if any)? 
    \item[] Answer: \answerNA{} 
    \item[] Justification: The research does not involve crowdsourcing or human subjects.

\item {\bf Institutional review board (IRB) approvals or equivalent for research with human subjects}
    \item[] Question: Does the paper describe potential risks incurred by study participants, whether such risks were disclosed to the subjects, and whether Institutional Review Board (IRB) approvals (or an equivalent approval/review based on the requirements of your country or institution) were obtained?
    \item[] Answer: \answerNA{} 
    \item[] Justification: The research does not involve human subjects.

\item {\bf Declaration of LLM usage}
    \item[] Question: Does the paper describe the usage of LLMs if it is an important, original, or non-standard component of the core methods in this research? Note that if the LLM is used only for writing, editing, or formatting purposes and does \emph{not} impact the core methodology, scientific rigor, or originality of the research, declaration is not required.
    \item[] Answer: \answerNA{} 
    \item[] Justification: The core methodology strictly relies on mathematical optimization and QAP solvers, not Large Language Models.

\end{enumerate}

\end{document}